\documentclass[twoside]{article}

\usepackage{aistats2022}
\usepackage{amsmath, amssymb, hyperref}
\usepackage{caption}
\usepackage{subcaption}
\usepackage{dsfont}
\usepackage{algorithm}
\usepackage[noend]{algpseudocode}
\usepackage{comment}
\usepackage{graphicx}
\usepackage{natbib}
\usepackage{amsthm}

\renewcommand{\phi}{\varphi}
\renewcommand{\epsilon}{\varepsilon}

\usepackage{amsthm}
\theoremstyle{definition}
\newtheorem{theorem}{Theorem}
\newtheorem{lemma}[theorem]{Lemma}
\newtheorem{remark}{Remark}
\newtheorem{claim}{Claim}
\newtheorem{corollary}{Corollary}

\DeclareRobustCommand{\divby}{%
  \mathrel{\text{\vbox{\baselineskip.65ex\lineskiplimit0pt\hbox{.}\hbox{.}\hbox{.}}}}%
}

\begin{document}

%

%

\twocolumn[

\aistatstitle{Multilayer Lookahead: a Nested Version of Lookahead}

\aistatsauthor{ Denys Pushkin \And Luis Barba }

\aistatsaddress{EPFL, Switzerland \\ \texttt{luis.barbaflores@epfl.ch} \And EPFL, Switzerland \\ \texttt{denys.pushkin@epfl.ch}}]

\begin{abstract}
In recent years, SGD and its variants have become the standard tool to train Deep Neural Networks. In this paper, we focus on the recently proposed variant Lookahead, which improves upon SGD in a wide range of applications. Following this success, we study an extension of this algorithm, the \emph{Multilayer Lookahead} optimizer, which recursively wraps Lookahead around itself. 
We prove the convergence of Multilayer Lookahead with two layers to a stationary point of smooth non-convex functions with $O(\frac{1}{\sqrt{T}})$ rate. 
We also justify the improved generalization of both Lookahead over SGD, and of Multilayer Lookahead over Lookahead, by showing how they amplify the implicit regularization effect of SGD.
We empirically verify  our results and show that Multilayer Lookahead outperforms Lookahead on CIFAR-10 and CIFAR-100 classification tasks, and on GANs training on the MNIST dataset. 
\end{abstract}

\section{INTRODUCTION}

Because of their simplicity, performance, and better generalization properties, stochastic variants of gradient descent have become the standard tool to train Deep Neural Networks. Experimental evidence~\citep{keskar2016large, hoffer2017train} shows that the use of smaller batch sizes further improves generalization, though it hinders the possibility of using parallelism. The development of new stochastic methods remains an active area of research. 
In particular, the Lookahead optimizer introduced by~\citet{zhang2019lookahead} has been shown to improve the performance of first-order stochastic methods for training deep neural networks. It also improves the convergence speed and the generalization of stochastic methods in numerous applications~\citep{zhang2019lookahead,chavdarova2020taming}.
In this paper, we further research around this optimizer. 

Lookahead is a wrapper around any other optimizer, which is called the \emph{base}, or \emph{inner optimizer}. It stores two states of the parameters, namely \emph{fast} and \emph{slow weights}. At each \emph{round}, it initializes fast weights to the current value of slow weights, then performs $k$ updates of the fast weights, using the inner optimizer, and assigns a new value of slow weights to be the convex combination of fast and slow weights with coefficients $\alpha$ and $1-\alpha$. Parameter $\alpha$ is called the \emph{slow learning rate}, or \emph{synchronization parameter}. 
Lookahead achieves faster convergence than its inner optimizer for classification tasks such as LSTM language models, transformers ~\citep{zhang2019lookahead} and GANs ~\citep{chavdarova2020taming}. 
Moreover, Lookahead is stable to suboptimal choice of hyperparameters ~\citep{zhang2019lookahead}.

However, a sufficient theoretical explanation of the success of Lookahead is still missing. 
\citet{zhang2019lookahead} claim that Lookahead can be considered as a method for variance reduction. They proved that Lookahead with SGD as the inner optimizer reduces the variance on the noisy quadratic loss function. \citet{wang2020lookahead} consider Lookahead from the perspective of multi-agent methods of optimization and show its convergence (with SGD as the inner optimizer) to a stationary point for smooth non-convex functions with $O(\frac{1}{ \sqrt{T}})$ rate, which matches the rate of SGD. Even though their analysis gives the same constant for the asymptotically most significant terms of SGD and Lookahead, it produces an additional term for Lookahead, which makes its convergence guarantee slightly worse~\citep{chavdarova2020taming}. However, in practice, we observe the opposite result. 

While Lookahead drastically improves the performance at the early stages of training, it usually produces only marginally better test loss after its convergence~\citep{zhang2019lookahead,chavdarova2020taming}.

Given the success of Lookahead in a wide range of applications, we go further and study a new extension, which we call the \emph{Multilayer Lookahead} optimizer. Succinctly, we can define it recursively as follows: Multilayer Lookahead with $n$ layers is a Lookahead optimizer, whose inner optimizer is Multilayer Lookahead with $n-1$ layers~\footnote{This algorithm for 2 layers was briefly discussed by~\citet{chavdarova2020taming}.}.

Our contributions are the following:

1) We introduce Multilayer Lookahead - an algorithm obtained by recursively iterating the Lookahead optimizer. We empirically verify that Multilayer Lookahead outperforms both its inner optimizer and Lookahead on image classification and GANs training tasks. 

2) We extend the analysis of~\citet{wang2020lookahead} for Lookahead and prove the convergence of the Lookahead with two layers with SGD as the inner optimizer towards a stationary point with rate $O(\frac{1}{\sqrt{T}})$. 

3) We prove that if the inner optimizer has a linear convergence rate, then both Lookahead and Multilayer Lookahead preserve the linear convergence rate (but, possibly, with a worse constant).

4) We propose a new explanation for the better generalization of Lookahead over SGD following the idea of implicit regularization and the backward analysis framework~\citep{smith2021origin, nichol2018first}. 

5) We show that Multilayer Lookahead exhibits a behavior similar to the ``super-convergence'' phenomenon studied by \citep{smith2019super}, where near top performance is achieved after just a few epochs.

\section{METHOD}

First, we formally revisit the original Lookahead algorithm~\citep{zhang2019lookahead} and then describe our extension of this method--- Multilayer Lookahead. 

Lookahead stores two parameters: fast weights $x_{r,i}$ and slow weights $y_r$. At each round $r$, it starts by initializing the fast weights $x_{r,0}$ to the current value of $y_r$, then performs $k$ updates of the fast weights using the inner optimizer, and obtains $x_{r,k}$. Finally, it performs the \emph{synchronization step} which moves the slow weights towards the final value of the fast weights by taking their convex combination: $y_{r+1} = (1-\alpha) y_{r} + \alpha x_{r,k}$.
The idea of using slow weights makes the trajectory of the optimizer smoother in a setting with a high variance of stochastic gradients.
See the pseudo-code in Algorithm \ref{Lookahead}.

\begin{algorithm}
\caption{\small Lookahead} \label{Lookahead}
\begin{algorithmic}[1]
\Require Loss function $f$; inner optimizer $A$; initial point $y_0$; synchronization parameter $\alpha$; length of the inner loop $k$; number of rounds $R$
\For {$r=0,1,\ldots R-1$}
\State $x_{r,0} = y_{r}$
\For {$i = 0, 1, \ldots k-1$}
\State Sample mini-batch $\xi_{r,i}$
\State $x_{r, i+1} = x_{r,i} - A(f_{\xi_{r,i}}, x_{r,i})$
\EndFor
\State $y_{r+1} = (1 - \alpha) y_r + \alpha x_{r,k}$ \Comment{synchronization}
\EndFor
\end{algorithmic}
\end{algorithm}

Comparing to the base optimizer, Lookahead needs extra space to store slow weights, which is at most two times more space in total. It requires also computing one additional convex combination per round, which has a negligible impact on time complexity.

We proceed now to formally introduce the Multilayer Lookahead optimizer. 
The easiest way to define it is recursive: Lookahead with $n$ layers is a Lookahead, whose inner optimizer is Lookahead with $n-1$ layers. 
Thus, it stores $n$ states of the parameters, has $n$ synchronization parameters $\alpha_1, \ldots, \alpha_n$, and $n$ numbers of steps per update $k_1, \ldots, k_n$, one per each layer. 
We provide a pseudo-code for Lookahead with two layers in Algorithm \ref{2-Lookahead}. 
Among all $n$ weights, we will often refer to the \emph{outer weights} (weights at level $1$; those that update the least frequently) and the \emph{inner weights} (weights at level $n$; those that update the most frequently). 
As before, by \emph{round} we mean one update of the outer variable, and by \emph{iteration} - one update of the inner variable, possibly followed by the synchronization with the variables from lower levels. 
One update of the outer variable requires computing $k_n$ updates of the variable from the level $2$, that in turn requires $k_n k_{n-1}$ updates of the variable from the level $3$ and so on. 
Hence, one update of the outer variable requires $k_n \ldots k_1$ updates of the inner variable, or, equivalently, one round consists of $k_n \ldots k_1$ iterations.
In the following, we will refer to $R$ as the number of rounds, and $T = R k_1 k_2 \ldots k_n$ as the total number of iterations.

Since the Lookahead with $n$ layers needs to store $n+1$ states of the parameters, it may increase the space complexity up to $n+1$ times. However, the additional time cost is still insignificant. Indeed, it requires the same number of gradients evaluation as the inner optimizer. Moreover, during each round it requires $1 + k_n + k_n k_{n-1} + \ldots + k_{n} k_{n-1}...k_2 < k_n k_{n-1} ... k_1$ synchronizations (under the natural assumption that $k_i \geq 2$ $\forall i$). Therefore, it requires less than 1 synchronization per iteration in average, independently of the number of layers.

\section{CONVERGENCE ANALYSIS} \label{sec:convergence_to_stationary}


\citet{wang2020lookahead} proved the convergence of Lookahead, which uses SGD as the base optimizer, to a stationary point for smooth, not necessarily convex functions $f$. Their analysis relies on considering Lookahead as a specific instance of local SGD (\citet{stich2018local, koloskova2020unified, woodworth2020local}), where two agents solve the composite minimization problem $min_{x,y} f_1(x) + f_2(y)$, subject to consensus constraint $x=y$ at the end of each round. Specifically, the first agent updates the fast weights by optimizing function $f_1 = f$, and the second one updates the slow weights by optimizing function $f_2 \equiv 0$. After each agent performs $k$ steps, they synchronize their weights by taking their convex combination. 

Following the same idea, we prove similar result for Lookahead with two layers. Let's represent 2-layers Lookahead as an instance of local SGD for a specific composite optimization problem. Consider problem of minimizing $f_1(x) + f_2(y) + f_3(z)$, where $f_1 = f$, $f_2 = f_3 \equiv 0$, subject to consensus constraints $x=y$, if $t \divby k_1$ (by $\divby$ we denote ``divisible by"), and $x=y=z$, if $t \divby k_1 k_2$, where $t$ is a number of local updates, or iterations. Denote by $x_t$, $y_t$, $z_t$ the values of the parameters and by $\gamma_t$ the learning rate at iteration $t$. Define parameters matrix $X_t = (x_t, y_t, z_t) \in \mathbb{R}^{d \times 3}$, matrix of stochastic gradients $G_t = (g(x_t, \xi_t), 0, 0) \in \mathbb{R}^{d \times 3}$, and synchronization matrix:
\begin{gather}
    \nonumber P_t = 
    \begin{cases}
        I_3, \: \text{if} \: t+1 \not \: \divby k_1 \\
        \begin{pmatrix}
            \alpha_1 & \alpha_1 & 0 \\
            1 - \alpha_1 & 1 - \alpha_1 & 0 \\
            0 & 0 & 1
        \end{pmatrix}, \: \text{if} \:
        t+1 \divby k_1, \\ t+1 \not \: \divby k_1 k_2 \\
        \begin{pmatrix}
            \alpha_1 \alpha_2 & \alpha_1 \alpha_2 & \alpha_1 \alpha_2 \\
            (1 - \alpha_1) \alpha_2 & (1 - \alpha_1) \alpha_2 & (1 - \alpha_1) \alpha_2 \\
            1 - \alpha_2 & 1 - \alpha_2 & 1 - \alpha_2
        \end{pmatrix}, \\ \text{if} \:
        t+1 \divby k_1 k_2
    \end{cases}
\end{gather}
and define the update rule as 
\begin{gather}\label{nine}
    X_{t+1} = (X_t - \gamma_t G_t) P_t
\end{gather}
with initial conditions $X_0 = (z_0, z_0, z_0)^T$. 
Then the sequence $(x_t, y_t, z_t, t\geq 0)$, produced by this instance of local SGD, coincides with the sequence $(x_{r,i,j}, y_{r,i}, z_r, t \geq 0)$, produced by 2-layers Lookahead applied to optimizing $f$, where $r,i,j$ are defined from $t$ using the equality $t = r k_1 k_2 + i k_1 + j$, $0 \leq i \leq k_2-1$, $0 \leq j \leq k_1-1$.
Indeed, these sequences start from the same initialization and have the same update rule. 
Thus, instead of directly proving the convergence of 2-layers Lookahead, we can show an equivalent statement for the considered local SGD instance.

\begin{algorithm}
\caption{\small Lookahead with two layers} \label{2-Lookahead}
\begin{algorithmic}[1]
\Require Loss function $f$; inner optimizer $A$; initial point $z_0$; synchronization parameters $\alpha_1$, $\alpha_2$; lengths of the inner loops $k_1$, $k_2$; number of rounds $R$
\For {$r=0,1,\ldots R-1$}
\State $y_{r,0} = z_{r}$
\For {$i = 0, 1, \ldots k_2-1$}
\State $x_{r,i,0} = y_{r,i}$
\For {$j = 0, 1, \ldots k_1-1$}
\State Sample mini-batch $\xi_{r,i,j}$
\State $x_{r,i,j+1} = x_{r,i,j} - A(f_{\xi_{r,i,j}}, x_{r,i,j})$
\EndFor
\State $y_{r, i+1} = (1 - \alpha_1) y_{r,i} + \alpha_1 x_{r,i,k_1}$
\EndFor
\State $z_{r+1} = (1 - \alpha_2) z_{r} + \alpha_2 y_{r, k_2}$
\EndFor
\end{algorithmic}
\end{algorithm}

Denote $a = (\alpha_1 \alpha_2, (1 - \alpha_1) \alpha_2, 1 - \alpha_2)^T \in \mathbb{R}^3$. It can be easily verified that $P_t a = a$ holds for any $t$, that is, $a$ is an eigenvector of $P_t$ with eigenvalue 1 for any $t$. Thus, multiplying (\ref{nine}) by $a$ from the right, we get:

\begin{equation}
    X_{t+1} a = X_t a - \gamma_t G_t a = X_t a - \gamma_t \alpha_1 \alpha_2 g(x_t, \xi_t). \label{conv stat: 1}
\end{equation}

Let $\theta_t = X_t a = \alpha_1 \alpha_2 x_t + (1 - \alpha_1) \alpha_2 y_t + (1 - \alpha_2) z_t$.
Using this notation, we can rewrite equation (\ref{conv stat: 1}) as:

\begin{equation}\label{theta}
    \theta_{t+1} = \theta_t - \gamma_t \alpha_1 \alpha_2 g(x_t, \xi_t)
\end{equation}

Intuitively, $\theta_t$ is a convex combination of $x_t$, $y_t$ and $z_t$ that evolves smoothly, regardless of the synchronizations.
Consequently, $\theta_t$ coincides with $x_t$, $y_t$ and $z_t$ at the end of each round.

As in ~\cite{wang2020lookahead}, we assume that the function $f$ is $L$-smooth, i.e. satisfies:
\begin{gather}\label{smoothness}
    \lVert\nabla f(x) - \nabla f(y)\rVert \leq L \lVert x - y \rVert, \forall x, y \in \mathbb{R}^d
\end{gather}
and that stochastic gradients are independent, unbiased, and have uniformly bounded variance:
\begin{gather}
    \label{sg:independent} \{g(x, \xi), x \in \mathbb{R}^d\} - \text{independent}\\
    \label{sg:unbiased} \mathbb{E}_{\xi} [g(x, \xi) | x] = \nabla f(x) \\
    \label{sg:bounded} \mathbb{E}_{\xi} [\lVert g(x, \xi) - \nabla f(x)\rVert^2 | x] \leq \sigma^2
\end{gather}
We will need one more technical assumption, which arises from the proof:
\begin{gather}
    \nonumber \forall t \geq 0: \quad
    1 - \gamma_{t} \alpha_1 \alpha_2 L - 2 \gamma_{t}^2 L^2 k_1^2 \times \\
    \times \Big((1-\alpha_1)^2 \alpha_2^2 + 2(1-\alpha_2)^2 + 2 \alpha_1^2 (1-\alpha_2)^2 k_2^2\Big) \geq 0, \label{gamma:constraint_2}
\end{gather}
Let us elaborate on the last assumption. It can be considered in the form $g(\gamma_t) \geq 0$, where $g(\gamma_t)$ is the left side of the inequality as a function of learning rate $\gamma_t$. Note that $g$ is continuous and strictly decreasing (for non-negative $\gamma_t$). Moreover, $g(0) = 1$ satisfies the inequality. Thus, the solution set is always non-empty and has the form:
\begin{gather}
    \gamma_t \leq \gamma_*, \: \forall t \geq 0 \label{gamma:constraint_2_alt}
\end{gather}
where $\gamma_* = g^{-1}(0)$ is implicitly defined from (\ref{gamma:constraint_2}). 

\begin{theorem}\label{theorem:1}

Suppose that learning rate is kept constant within each round: $\gamma_{r k_1 k_2 + i k_1 + j} = \gamma_{r k_1 k_2}$, $\forall r \geq 0$, $0 \leq i \leq k_2-1$, $0 \leq j \leq k_1-1$, and satisfies the following conditions:
\begin{gather}
    \lim_{r \rightarrow \infty} \gamma_{r k_1 k_2} = 0, \quad \sum_{r=0}^{\infty} \gamma_{r k_1 k_2} = \infty \label{gamma:constraint_1}
\end{gather}
Then under the assumptions (\ref{smoothness}) - (\ref{gamma:constraint_2}) we have:
\begin{gather*}
    \frac{1}{k_1 k_2 \sum_{r=0}^{R-1} \gamma_{r k_1 k_2}} \sum_{r=0}^{R-1} \\
    \Big(\gamma_{r k_1 k_2} \sum_{i=0}^{k_2-1} \sum_{j=0}^{k_1 - 1} \mathbb{E}[\lVert\nabla f(\theta_{r k_1 k_2+i k_1+j})\rVert^2]\Big) \rightarrow 0, r \rightarrow \infty
\end{gather*}
i.e., the weighted average of the expected squared norms of the gradients approaches to 0.

\end{theorem}

Recall that $T = R k_1 k_2$ is the total number of iterations of 2-layers Lookahead.

\begin{theorem}\label{theorem:2}
Suppose that the learning rate is kept constant, i.e. $\gamma_t = \gamma \: \forall t\geq0$. Then under the assumptions (\ref{smoothness}) - (\ref{gamma:constraint_2}), we have:
\begin{gather}
    \nonumber \frac{1}{T} \sum_{t=0}^{T-1} \mathbb{E}[\lVert\nabla f(\theta_{t})\rVert^2] \leq 
    \frac{2(f(\theta_0) - f_{inf})}{\gamma \alpha_1 \alpha_2  T} + \\
    \nonumber + \gamma \alpha_1 \alpha_2 L \sigma^2 + 2 \gamma^2 L^2 \sigma^2 k_1 \times \\
    \times \left( (1-\alpha_1)^2 \alpha_2^2 + 2(1-\alpha_2)^2 + \frac{4}{3} \alpha_1^2 (1-\alpha_2)^2 k_2^2 \right) \label{stat conv theorem:2 formula}
\end{gather}
where $f_{inf}$ denotes the infimum of the objective function. Further, by setting $\gamma = \frac{\gamma_*}{\sqrt{T}}$, where $\gamma_*$ is defined in (\ref{gamma:constraint_2_alt}), we get the following bound:
\begin{gather}
    \nonumber \frac{1}{T} \sum_{t=0}^{T-1} \mathbb{E}[\lVert\nabla f(\theta_{t})\rVert^2] \leq \\
    \nonumber \leq \frac{2(f(\theta_0) - f_{inf}) + \gamma_*^2 \alpha_1^2 \alpha_2^2 L \sigma^2}{\gamma_* \alpha_1 \alpha_2 \sqrt{T}} 
    + \mathcal{O}\left(\frac{1}{T}\right) = \\
    = \mathcal{O}\left(\frac{1}{\sqrt{T}}\right) \label{corrolary:1 formula}
\end{gather}
\end{theorem}

For the proof of all statements in this section, see Appendix~\ref{proofs section 3}.

\begin{claim}\label{conv_analysis: claim 1}
If we optimize the bound on the right side of (\ref{stat conv theorem:2 formula}) over $\gamma$ precisely, respecting the constraint (\ref{gamma:constraint_2}), then for large $T$ the best bound will be achieved for $\alpha_1 = \alpha_2 = 1$, that is, when 2-layers Lookahead degenerates to SGD.
\end{claim}

Thus, as in the analysis of Lookahead by ~\cite{wang2020lookahead}, our result, while provides some theoretical guarantees for 2-layers Lookahead, does not capture the improvement of this method over SGD.

\begin{corollary} \label{corollary:2}
If $T \geq \frac{2(f(\theta_0) - f_{inf})}{\alpha_1^2 \alpha_2^2 L \sigma^2 \gamma_*^2}$, then the learning rate $\gamma = \frac{1}{\alpha_1 \alpha_2 \sigma} \sqrt{\frac{2(f(\theta_0) - f_{inf})}{T L}}$ satisfies the constraint (\ref{gamma:constraint_2_alt}) and gives the improved bound:
\begin{gather*}
    \frac{1}{T} \sum_{t=0}^{T-1} \mathbb{E}[\lVert\nabla f(\theta_{t})\rVert^2] \leq \\
    \leq \frac{2\sigma \sqrt{2 L (f(\theta_0) - f_{inf})}}{\sqrt{T}} + \mathcal{O}\left(\frac{1}{T}\right)
     = \mathcal{O}\left(\frac{1}{\sqrt{T}}\right)
\end{gather*}

Furthermore, the obtained constant for the asymptotically most significant term $\frac{1}{\sqrt{T}}$ is the best possible, that can be deduced from \autoref{theorem:2}.
\end{corollary}

Let us make a few observations about the result of Corollary \ref{corollary:2}. 
First, note that the optimal constant for the leading term $\frac{2\sigma \sqrt{2 L (f(\theta_0) - f_{inf})}}{\sqrt{T}}$ does not depend on $\alpha_1$ and $\alpha_2$. 
Hence, our result provides almost the same optimal bound for SGD and 2-layers Lookahead, where the difference comes only through the lower-order $\mathcal{O}\left(\frac{1}{T}\right)$ term.

Second, the proposed learning rate $\gamma = \frac{1}{\alpha_1 \alpha_2 \sigma} \sqrt{\frac{2(f(\theta_0) - f_{inf})}{T L}} = \Theta\left(\frac{1}{\alpha_1\alpha_2\sqrt{T}}\right)$ depends reciprocally on synchronization parameters $\alpha_1$ and $\alpha_2$.

Finally, Theorem \ref{theorem:2} guarantees $\frac{1}{\sqrt{T}}$ convergence rate only for a fixed number of iterations $T$ which is known in advance (since the proposed learning rate $\gamma = \Theta\left(\frac{1}{\sqrt{T}}\right)$ depends on $T$). 
To obtain $\frac{1}{\sqrt{T}}$ rate for $T \rightarrow \infty$, we can combine 2-layers Lookahead with restarts (see Theorem \ref{theorem:3} in Appendix).

We highlight one more property of Multilayer Lookahead: if the inner optimizer has a linear convergence rate, then Multilayer Lookahead preserves it (possibly with a worse constant). For details, see Appendix \ref{MLA linear rate}.

\section{GENERALIZATION ANALYSIS}

The study of optimizers is not limited to their convergence properties. In Machine Learning, our final aim is to perform well on the new, unseen data. Optimizer parameters have a crucial effect on generalization. For instance, mini-batch SGD generalizes better than full-gradient descent, with smaller mini-batches contributing to even better generalization~\citep{keskar2016large, hoffer2017train}. The large learning rate also plays a regularization role \citep{lewkowycz2020large, li2019towards}. To theoretically analyze such effects, we need a tool that captures the generalization property. \cite{smith2021origin} have shown that both the effect of small mini-batch and large learning rate can be explained by the implicit regularizer of an optimizer.

Let us consider a first-order method for minimizing function $f(x)$.
When the step size approaches zero, the trajectory of the optimizer converges to the trajectory of its continuous version. The latter can be expressed as the solution of the corresponding ODE.
For example, trajectories of both GD and SGD converge to the continuous solution of the following ODE~\citep{yaida2018fluctuation}: $\dot{x} = - \nabla f(x), \: x(0) = x_0$.

However, for any constant step size, an optimizer makes some consistent skews from its continuous version. To analyze these skews, \cite{smith2021origin} construct a \emph{modified flow}, which depends on the learning rate, such that the iterations of the optimizer, applied to the function $f$, lie on the trajectory of the ODE for the function $\Tilde{f}$. Then, while applying optimizer to the function $f$, the optimizer iterations follow the ODE for the modified flow $\Tilde{f}$. Thus, the difference $\Tilde{f} - f$ is called the \emph{implicit regularizer} of a given optimizer. In the case of convergence, the limiting point of an optimizer can be described as a critical point of the modified flow $\Tilde{f}$, and we can leverage the form of $\Tilde{f}$ to explain the generalization properties of the optimizer.

\subsection{Previous Results}

For the full-gradient descent (GD) with the learning rate $\gamma$, the modified flow has the following form ~\citep{smith2021origin, barrett2020implicit}:
\begin{gather}
    \Tilde{f}_{GD}(y) = f(y) + \frac{\gamma}{4} \lVert \nabla f(y) \rVert^2 + \mathcal{O}(\gamma^2) \label{generalization: 20}
\end{gather}
It indicates that GD implicitly penalizes sharp regions, where the gradient norm is large. 

To derive the modified flow for SGD with mini-batches, \citet{smith2021origin} assumed that the samples are split into a fixed set of mini-batches so that each sample belongs to exactly one mini-batch, and only the order of mini-batches shuffles randomly. Then, the average iterate of SGD at the end of the epoch, averaged across all possible permutations of the mini-batches, follows the ODE for the modified flow:
\begin{gather}
    \Tilde{f}_{SGD}(y) =
    f(y) + \frac{\gamma}{4 m} \sum_{i=0}^{m-1} \lVert \nabla f_i(y) \rVert^2 + \mathcal{O}(\gamma^2) \label{generalization: 23}
\end{gather}
or, equivalently:
\begin{gather}
    \nonumber \Tilde{f}_{SGD}(y) = 
    f(y) + \frac{\gamma}{4} \lVert \nabla f(y) \rVert^2 + \\
    + \frac{\gamma}{4 m} \sum_{i=0}^{m-1} \lVert \nabla f_i(y) - \nabla f(y) \rVert^2 + \mathcal{O}(\gamma^2) \label{generalization: 21}
\end{gather}
From (\ref{generalization: 21}) we see that SGD penalizes not only sharp regions, but also non-uniform regions (with a high variance of stochastic gradients). ~\cite{smith2021origin} exploits this fact to explain the improved generalization of SGD over GD. In section Section \ref{sec: GD vs SGD vs LA}, we will provide an alternative explanation of this phenomenon.

The analysis made by ~\cite{smith2021origin} still works in more a general setting, where instead of fixing the set of mini-batches, one randomly shuffles the samples across the mini-batches at the beginning of every epoch so that each sample belongs to exactly one mini-batch. In this case, by replacing the sum over all mini-batches by the expectation with a corresponding scaling factor everywhere in the proof, we get the alternative form of (\ref{generalization: 23}):
\begin{gather}
    \Tilde{f}_{SGD}(y) = f(y) + \frac{\gamma}{4} \mathbb{E}[\lVert \nabla f_0(y) \rVert^2] + \mathcal{O}(\gamma^2) \label{generalization: 22}
\end{gather}
To construct the modified flow for GD and SGD, ~\cite{smith2021origin} use a \emph{backward error analysis}. 
Here, we briefly summarize their approach, and we will apply it to derive the modified flow for Lookahead in the next section. 

Consider the ODE for the modified flow $\dot{y} = - \nabla \Tilde{f}(y), y(0) = y_0$ in general form:
\begin{gather}
    \dot{y} = \Tilde{h}(y), y(0) = y_0 \label{generalization: 14}
\end{gather}
Denote the iterations of an optimizer with a learning rate $\gamma$, which starts from the same initial point $y_0$, by $y_1, y_2, \ldots$.
We aim to find such $\Tilde{h}$ that the average final iterate of the epoch $\mathbb{E}[y_m]$ lies on the trajectory of $y$.
Let us searched for $\Tilde{h}$ in the following form:
\begin{gather}
    \Tilde{h}(y) = h(y) + \epsilon h_1(y) + \epsilon^2 h_2(y) + \ldots \label{generalization: 15}
\end{gather} 
\citet{smith2021origin} established that for such $\Tilde{h}$  the solution of (\ref{generalization: 14}) satisfies:
\begin{gather}
    \nonumber y(\epsilon) = y_0 + \epsilon h(y_0) + \\
    + \epsilon^2 \left( h_1(y_0) + \frac{1}{2} \nabla h(y_0) h(y_0) \right) + \mathcal{O}(\epsilon^3) \label{generalization: 13}
\end{gather}
It remains to find $h$ and $h_1$ from the condition $y(\epsilon) = \mathbb{E}[y_m]$ with $\epsilon = \gamma m$. 
After this, we can reconstruct the modified flow $\Tilde{f}$ from the condition $\Tilde{h} = - \nabla \Tilde{f}$. 
For more details, see \citet{smith2021origin}.

\subsection{Implicit Regularizer of Lookahead} \label{sec: implicit_reg_LA}

In this section, we apply the backward error analysis to derive the modified flow of the Lookahead (LA).

Before moving to our main result, let us introduce the auxiliary functions:
\begin{gather}
    AN(y) = \mathbb{E}[\lVert \nabla f_0(y) \rVert^2] = \mathbb{E}[\lVert \nabla f_i(y) \rVert^2] \: \forall i \\
    \nonumber AI(y) = \mathbb{E}[\langle \nabla f_0(y), \nabla f_1(y) \rangle] = \\
    = \mathbb{E}[\langle \nabla f_i(y), \nabla f_j(y) \rangle] \: \forall i \neq j \\
    ANG(y) = \nabla AN(y) \\
    AIG(y) = \nabla AI(y)
\end{gather}
Here, $AN$ and $AI$ stand for Average Norm squared and Averaged Inner product, while $ANG$ and $AIG$ stand for Average Norm squared Gradient and Averaged Inner product Gradient correspondingly. 
The expectation is taken across all possible splittings on the mini-batches.
With these notations, we can rewrite the modified flow for SGD (\ref{generalization: 22}) as follows:
\begin{gather}
    \Tilde{f}_{SGD}(y) = f(y) + \frac{\gamma}{4} AN(y) + \mathcal{O}(\gamma^2) \label{generalization: 18}
\end{gather}
Let us also rewrite (\ref{generalization: 20}) in terms of $AN(y)$ and $AI(y)$.
\begin{gather*}
    \nonumber \lVert \nabla f(y) \rVert^2 = 
    \lVert \frac{1}{m} \sum_{i=0}^{m-1} \nabla f_i(y) \rVert^2 = \\
    = \frac{1}{m^2} \left( \sum_{i=0}^{m-1} \lVert \nabla f_i(y) \rVert^2 + \sum_{i' \neq i} \langle \nabla f_{i'}(y), \nabla f_i(y) \rangle \right)
\end{gather*}
Taking the expectation of both sides, we get:
\begin{gather}
    \nonumber \lVert \nabla f(y) \rVert^2 =
    \frac{1}{m^2} \left( m AN(y) + m(m-1) AI(y) \right) = \\
    = \frac{1}{m} \left( AN(y) + (m-1) AI(y) \right)
\end{gather}
Thus, the formula for $\Tilde{f}_{GD}(y)$ becomes
\begin{gather}
    \nonumber \Tilde{f}_{GD}(y) = 
    f(y) + \frac{\gamma}{4} \frac{1}{m} \left( AN(y) + (m-1) AI(y) \right) =\\
    = f(y) + \frac{\gamma}{4 m} AN(y) + \frac{\gamma (m-1)}{4 m} AI(y)
\end{gather}
Now we are ready to present the modified flow of the Multilayer Lookahead optimizer.

\begin{theorem} \label{theorem:impl reg LA}

Consider $n$-layers Lookahead with parameters $\{k_1, \ldots, k_n\}$, $\{\alpha_1, \ldots, \alpha_n\}$ and SGD with learning rate $\gamma$ as the inner optimizer. 
Suppose that the number of mini-batches per epoch $m$ is divisible by the number of iterations per one round $k_1 \ldots k_n$. 
Besides, assume that at each epoch the samples are randomly shuffled across the mini-batches such that each sample belongs to exactly one mini-batch. 
Then the final iteration of the epoch for the $n$-layers Lookahead, averaged across the randomness in splitting on the mini-batches, lies on the trajectory of the ODE for the following modified flow:
\begin{gather}
    \nonumber \Tilde{f}_{LA-n}(y) = f(y)
    + \frac{\gamma \alpha_1 \ldots \alpha_n}{4} AN(y) - \\
    \nonumber - \frac{\gamma}{4} \sum_{p=1}^{n} (1-\alpha_p) \alpha_{p-1} \ldots \alpha_1 (k_p \ldots k_1 - 1) AI(y) + \\
    + \mathcal{O}(\gamma^2)
\end{gather}
Specifically, the solution of the ODE $\dot{y} = - \nabla \Tilde{f}_{LA-n}(y)$, $y(0) = y_0$ satisfies $y(\alpha_1 \ldots \alpha_n \gamma m) = \mathbb{E}[y_r]$, where $r = m / (k_1 \ldots k_n)$, so that $y_r$ - output of the Lookahead after one epoch (or $r$ rounds), starting from the point $y_0$.

\end{theorem}

For the proof, see Appendix \ref{proof theorem 4}. 

Let us explicitly state the formulas for LA and LA-2:
\begin{gather}
    \Tilde{f}_{LA}(y) = f(y) + \frac{\alpha \gamma}{4} AN(y) - \frac{\gamma}{4} (1-\alpha)(k-1) AI(y) \\
    \nonumber \Tilde{f}_{LA-2}(y) = f(y) + \frac{\alpha_1 \alpha_2 \gamma}{4} AN(y) - \\
    - \frac{\gamma}{4} ((1-\alpha_1)(k_1-1) + (1-\alpha_2)\alpha_1 (k_2 k_1 - 1)) AI(y)
\end{gather} 
\subsection{Comparing the Implicit Regularizers for GD, SGD and LA} \label{sec: GD vs SGD vs LA}

Let us summarize the modified flows for the considered optimizers, assuming that each optimizer uses its own learning rate $\gamma_{GD}$, $\gamma_{SGD}$, and $\gamma_{LA}$ (we omit $\mathcal{O}(\gamma^2)$ term everywhere for simplicity):
\begin{gather*}
    \Tilde{f}_{GD}(y) = f(y) + \frac{\gamma_{GD}}{4 m} AN(y) + \frac{\gamma_{GD} (m-1)}{4 m} AI(y) \\
    \Tilde{f}_{SGD}(y) = f(y) + \frac{\gamma_{SGD}}{4} AN(y) \\
    \Tilde{f}_{LA}(y) = f(y) + \frac{\alpha \gamma_{LA}}{4} AN(y) - \\
    - \frac{(1-\alpha) \gamma_{LA} (k-1)}{4} AI(y)
\end{gather*}
To ensure that all optimizers make the same progress during one epoch,
the sum of step sizes during one epoch has to be the same: 
$\gamma_{GD} = m \gamma_{SGD} = m \alpha \gamma_{LA}$ 
(in general, each layer of Lookahead shrinks the progress of the outer variable by a factor of $\alpha_i$). 
Then, the coefficient before $AN(y)$ is equal for all three optimizers, 
and the difference comes only through the coefficient before $AI(y)$: it decreases in order GD, SGD, LA. 
Notably, we empirically observe that generalization improves in the same order.
Thus, we claim that the generalization improvement of SGD over GD and of LA over SGD caused by the amplifying role of the implicit regularizer $- AI(y) = - \mathbb{E}[\langle \nabla f_0(y), \nabla f_1(y) \rangle]$. 
Intuitively, by minimizing $-AI(y)$, we maximize the average inner product between the gradients on different mini-batches, which in turn minimizes the angle between these gradients. 
Thus, the regularizer $-AI(y)$ makes stochastic gradients more aligned.

From the formula for $\Tilde{f}_{LA}(y)$, we see that by increasing $k$, we also enhance the constant before -$AI(y)$, thus, increase the regularization role of Lookahead. For $\alpha$, we cannot conclude: by increasing $\alpha$, we amplify $AN(y)$ but suppress $-AI(y)$ regularizer.

We now compare implicit regularizers of LA and LA-2:
\begin{gather*}
    \Tilde{f}_{LA}(y) = f(y) + \frac{\alpha \gamma_{LA}}{4} AN(y) - \\
    - \frac{(1-\alpha) \gamma_{LA} (k-1)}{4} AI(y) \\
    \Tilde{f}_{LA-2}(y) = f(y) + \frac{\alpha_1 \alpha_2 \gamma_{LA-2}}{4} AN(y) - \\
    \frac{\gamma_{LA-2}}{4} ((1-\alpha_1)(k_1-1) + (1-\alpha_2)\alpha_1 (k_2 k_1 - 1)) AI(y)
\end{gather*}
Let us assume that their coefficients are related by $\alpha_1 = \alpha$, $k = k_1$, and the optimizers make equal progress during one epoch.
Then $\alpha \gamma_{LA} = \alpha_1 \alpha_2 \gamma_{LA-2}$, $0 < \alpha_2 < 1$ $\Rightarrow$ $\gamma_{LA-2} > \gamma_{LA}$. 
Then the constants before $AN(y)$ again coincide.
For $AI(y)$, even the first term $\frac{\gamma_{LA-2}}{4} (1-\alpha_1)(k_1-1)$ for LA-2 exceeds the whole constant for LA (since $\gamma_{LA-2} > \gamma_{LA}$).
Thus, by adding the second layer of Lookahead, we reinforce -$AI(y)$ regularizer further.
The same intuition works for any number of layers.

Note that scaling the learning rate for Lookahead by a factor of $\frac{1}{\alpha}$ and for 2-layers Lookahead by a factor of $\frac{1}{\alpha_1 \alpha_2}$ also arises in Corollary \ref{corollary:2}, which provides an asymptotically optimal learning rate for a convergence bound given by Theorem \ref{theorem:2}. Besides, we will see in the experimental section that the optimal learning rate is often greater for Multilayer Lookahead than for SGD. Thus, the assumption that optimizers make the same progress is theoretically justified and not detached from the practice.

\section{EXPERIMENTS}\label{sec: experiments}

\begin{figure}[h]
     \centering
     \hfill
     \begin{subfigure}[b]{0.45\textwidth}
         \includegraphics[width=\textwidth]{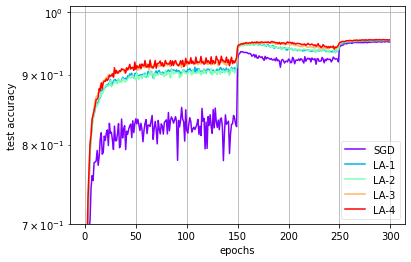}
         \caption{\small CIFAR-10} \label{cifar10 plot}
     \end{subfigure}
     \hfill
     \begin{subfigure}[b]{0.45\textwidth}
                  \includegraphics[width=\textwidth]{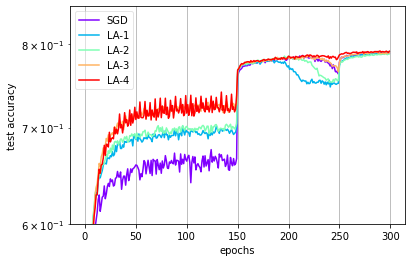}
         \caption{\small CIFAR-100} \label{cifar100 plot}
     \end{subfigure}

    \caption{\small \small Test Accuracy During the Training on (\ref{cifar10 plot}) CIFAR-10 and (\ref{cifar100 plot}) CIFAR-100. LA-$n$ Stands for Multilayer Lookahead with $n$ Layers}
    \label{fig:CIFAR-10 training plots}
\end{figure}


CIFAR-10 and CIFAR-100 datasets consist of $32 \times 32$ color images of 10 and 100 different classes correspondingly, split into 50000 training samples and 10000 test samples. 
We trained ResNet-18 \citep{he2016deep} model for 300 epochs with batch size = 50 and using learning rate decay by a factor of 10 after 150th and 250th epochs.
We selected this batch size to ensure that each epoch ends after the synchronization step in Lookahead, which makes the training plots less noisy (as shown by \citet{zhang2019lookahead}, Lookahead achieves the highest train and test accuracy right after the synchronization step).
As the inner optimizer, we used SGD with momentum = 0.9. 
We compared SGD\footnote{From now, we refer to SGD with momentum as SGD.} with Lookahead and Multilayer Lookahead with up to four layers.
Note that when using grid search, the number of possible configurations of $\alpha = (\alpha_1, \ldots, \alpha_n)$ and $k = (k_1, \ldots, k_n)$ increases exponentially with the number of layers.
Thus, we tried a moderate range of values for $\alpha$ and $k$ and do not claim that the best parameters that we found are close to optimal.
We empirically noticed that for any number of Lookahead layers the best parameter $\alpha = (\alpha_1, \ldots, \alpha_n)$ is always achieved for a non-increasing sequence of $\alpha_i$ (that is, when Multilayer Lookahead gives more preference to the weight from inner loops during the synchronization). 
This observation enables to significantly reduce the tuning time for $\alpha$. 

For CIFAR-10, we performed a grid search for SGD over learning rate $lr \in \{0.05, 0.1, 0.2, 0.3, 0.5\}$ and weight decay parameter $\lambda \in \{10^{-3}, 3\cdot 10^{-4}, 10^{-4}\}$, and got $lr = 0.1$, $\lambda=3\cdot 10^{-4}$ performing the best. 
For CIFAR-100, we tuned SGD over $lr \in \{0.003, 0.005, 0.01, 0.03, 0.05, 0.1, 0.2, 0.3\}$ and $\lambda \in \{3\cdot 10^{-3}, 10^{-3}, 3\cdot 10^{-4}, 10^{-4}\}$, and found $lr = 0.01$, $\lambda=10^{-3}$ to be the best. 
The set of parameters that we tried for Multilayer Lookahead and additional training plots are presented in Appendix \ref{cifar experiments details}.

You can see the performance of each optimizer using the best parameters found in Table \ref{tab: CIFAR-10 acc_test} (CIFAR-10) and Table \ref{tab: CIFAR-100 acc_test} (CIFAR-100), and the corresponding training plots in Figure \ref{fig:CIFAR-10 training plots} (the results were averaged across 3 runs). 
They show that even for a moderate grid search, Multilayer Lookahead often outperforms both SGD and Lookahead. 
Moreover, as we can see from the plots, in both cases, Multilayer Lookahead with the most layers achieves better anytime performance. 
We note that the performance of Multilayer Lookahead was robust towards the choice of parameters ($\alpha$, $k$) in the considered range.

\begin{table}[t]
\caption{\small Results on CIFAR-10} \label{tab: CIFAR-10 acc_test}
\begin{center}
\begin{tabular}{c | c | c | c }
Opt & $lr$ & $\alpha$ & test accuracy \\
\hline \\
SGD & 0.1 & - & $0.9527 \pm 0.0019$ \\
LA-1 & 0.1 & 0.5 ($k$=10) & $0.9541 \pm 0.0009$ \\
LA-2 & 0.1 & (0.7, 0.7) & \textbf{0.9555} $\pm$ 0.0016 \\
LA-3 & 0.1 & (0.6, 0.8, 0.8)  & $0.9553 \pm 0.0021$ \\
LA-4 & 0.1 & (0.6, 0.75, 0.85, 0.85) & $0.9554 \pm 0.0011$  \\
\end{tabular}
\end{center}
\end{table}

\begin{table}[h]
\caption{\small Results on CIFAR-100} \label{tab: CIFAR-100 acc_test}
\begin{center}
\begin{tabular}{c | c | c | c }
Opt & $lr$ & $\alpha$& test accuracy \\
\hline \\
SGD & 0.01 & - & $0.7904 \pm 0.0032$ \\
LA-1 & 0.05 & 0.3 ($k$=5)& $0.7910 \pm 0.0010$ \\
LA-2 & 0.03 & (0.5, 0.8) & $0.7885 \pm 0.0022$ \\
LA-3 & 0.03 & (0.6, 0.75, 0.75) & $0.7910 \pm 0.0020$ \\
LA-4 & 0.03 & (0.6, 0.75, 0.8, 0.8) & \textbf{0.7923} $\pm$ 0.0020 \\
\end{tabular}
\end{center}
\end{table}

For CIFAR-100, we found the optimal learning rate for Multilayer Lookahead to be larger than for SGD (Table \ref{tab: CIFAR-100 acc_test}).
In this case, we clearly see the regularization effect of Multilayer Lookahead: comparing to SGD, it performs worse on training data, but better on test data (see Appendix \ref{cifar experiments details} for the corresponding plots).
In contrast, for CIFAR-10, the optimal learning rate for SGD and Multilayer Lookahead turned out to be the same. In this case, Multilayer Lookahead improves both train and test accuracy.
These observations are consistent with our findings in Section \ref{sec: GD vs SGD vs LA}, where we established that Multilayer Lookahead with a larger learning rate amplifies the regularization effect of SGD. 
However, when the learning rates are equal, we cannot conclude. Even though Multilayer Lookahead still yields an additional -$AI(y)$ regularizer, it suppresses $AN(y)$ regularizer.

\paragraph{Super-convergence.}
In all of our experiments, we found that the network can be trained with Multilayer Lookahead to almost top accuracy orders of magnitude faster than with similar techniques. After just 30 epochs, we can reach 90\% test accuracy for CIFAR-10 and 70\% for CIFAR-100. This accuracies can be achieved before the first learning rate decay. Thus, in settings were learning rate decay could not be used, Multilayer Lookahead outperforms both SGD and Lookahead by a large margin.

For additional experiments training GANs on the MNIST dataset see Appendix~\ref{GANs on MNIST}.

\section{CONCLUSION}

In this work, we present Multilayer Lookahead - a method obtained by recursively wrapping Lookahead around any other optimizer. 
It increases space complexity linearly in the number of layers. 
However, the time overhead is marginal: on average, it requires taking less than one additional convex combination per iteration.
We prove the convergence of Lookahead with two layers to a stationary point for smooth non-convex functions with $\mathcal{O}\left( \frac{1}{\sqrt{T}} \right)$ rate. 
Besides, we show that Multilayer Lookahead amplifies the implicit regularizer yielded by SGD.
Next, we demonstrate that our optimizer outperforms both SGD and Lookahead on the classification task on CIFAR-10 and CIFAR-100 datasets.

Our experiments show a clear advantage over the use of SGD or Lookahead. While there is an extra cost in tuning the synchronization parameters, we have found the algorithm to be quite robust to the choice of these parameters. Indeed, having all synchronization parameters being equal leads to almost optimal results in all of our experiments. 

We also want to highlight that as it can be seen in our plots for test accuracy, the use of Multilayer Lookahead leads to a similar phenomenon as the so called ``super-convergence''~\citet{smith2019super}. That is, only a few epochs are needed to achieve an almost top accuracy. In particular, before the first learning rate decay, Multilayer Lookahead completely outperforms both SGD and Lookahead. This can be quite relevant when the amount of training data is not known in advance and a learning rate decay cannot be scheduled, e.g., in the streaming model. 

\newpage
\small
\bibliographystyle{iclr2020_conference}
\bibliography{main}

\begin{thebibliography}{20}
\providecommand{\natexlab}[1]{#1}
\providecommand{\url}[1]{\texttt{#1}}
\expandafter\ifx\csname urlstyle\endcsname\relax
  \providecommand{\doi}[1]{doi: #1}\else
  \providecommand{\doi}{doi: \begingroup \urlstyle{rm}\Url}\fi

\bibitem[Barrett \& Dherin(2020)Barrett and Dherin]{barrett2020implicit}
David~GT Barrett and Benoit Dherin.
\newblock Implicit gradient regularization.
\newblock \emph{arXiv preprint arXiv:2009.11162}, 2020.

\bibitem[Chavdarova et~al.(2020)Chavdarova, Pagliardini, Stich, Fleuret, and
  Jaggi]{chavdarova2020taming}
Tatjana Chavdarova, Matteo Pagliardini, Sebastian~U Stich, Francois Fleuret,
  and Martin Jaggi.
\newblock Taming gans with lookahead-minmax.
\newblock \emph{arXiv preprint arXiv:2006.14567}, 2020.

\bibitem[Goodfellow et~al.(2014)Goodfellow, Pouget-Abadie, Mirza, Xu,
  Warde-Farley, Ozair, Courville, and Bengio]{goodfellow2014generative}
Ian~J Goodfellow, Jean Pouget-Abadie, Mehdi Mirza, Bing Xu, David Warde-Farley,
  Sherjil Ozair, Aaron Courville, and Yoshua Bengio.
\newblock Generative adversarial networks.
\newblock \emph{arXiv preprint arXiv:1406.2661}, 2014.

\bibitem[He et~al.(2016)He, Zhang, Ren, and Sun]{he2016deep}
Kaiming He, Xiangyu Zhang, Shaoqing Ren, and Jian Sun.
\newblock Deep residual learning for image recognition.
\newblock In \emph{Proceedings of the IEEE conference on computer vision and
  pattern recognition}, pp.\  770--778, 2016.

\bibitem[Hoffer et~al.(2017)Hoffer, Hubara, and Soudry]{hoffer2017train}
Elad Hoffer, Itay Hubara, and Daniel Soudry.
\newblock Train longer, generalize better: closing the generalization gap in
  large batch training of neural networks.
\newblock \emph{arXiv preprint arXiv:1705.08741}, 2017.

\bibitem[Keskar et~al.(2016)Keskar, Mudigere, Nocedal, Smelyanskiy, and
  Tang]{keskar2016large}
Nitish~Shirish Keskar, Dheevatsa Mudigere, Jorge Nocedal, Mikhail Smelyanskiy,
  and Ping Tak~Peter Tang.
\newblock On large-batch training for deep learning: Generalization gap and
  sharp minima.
\newblock \emph{arXiv preprint arXiv:1609.04836}, 2016.

\bibitem[Kingma \& Ba(2014)Kingma and Ba]{kingma2014adam}
Diederik~P Kingma and Jimmy Ba.
\newblock Adam: A method for stochastic optimization.
\newblock \emph{arXiv preprint arXiv:1412.6980}, 2014.

\bibitem[Koloskova et~al.(2020)Koloskova, Loizou, Boreiri, Jaggi, and
  Stich]{koloskova2020unified}
Anastasia Koloskova, Nicolas Loizou, Sadra Boreiri, Martin Jaggi, and Sebastian
  Stich.
\newblock A unified theory of decentralized sgd with changing topology and
  local updates.
\newblock In \emph{International Conference on Machine Learning}, pp.\
  5381--5393. PMLR, 2020.

\bibitem[Lewkowycz et~al.(2020)Lewkowycz, Bahri, Dyer, Sohl-Dickstein, and
  Gur-Ari]{lewkowycz2020large}
Aitor Lewkowycz, Yasaman Bahri, Ethan Dyer, Jascha Sohl-Dickstein, and Guy
  Gur-Ari.
\newblock The large learning rate phase of deep learning: the catapult
  mechanism.
\newblock \emph{arXiv preprint arXiv:2003.02218}, 2020.

\bibitem[Li et~al.(2019)Li, Wei, and Ma]{li2019towards}
Yuanzhi Li, Colin Wei, and Tengyu Ma.
\newblock Towards explaining the regularization effect of initial large
  learning rate in training neural networks.
\newblock \emph{arXiv preprint arXiv:1907.04595}, 2019.

\bibitem[Nichol et~al.(2018)Nichol, Achiam, and Schulman]{nichol2018first}
Alex Nichol, Joshua Achiam, and John Schulman.
\newblock On first-order meta-learning algorithms.
\newblock \emph{arXiv preprint arXiv:1803.02999}, 2018.

\bibitem[Radford et~al.(2015)Radford, Metz, and
  Chintala]{radford2015unsupervised}
Alec Radford, Luke Metz, and Soumith Chintala.
\newblock Unsupervised representation learning with deep convolutional
  generative adversarial networks.
\newblock \emph{arXiv preprint arXiv:1511.06434}, 2015.

\bibitem[Salimans et~al.(2016)Salimans, Goodfellow, Zaremba, Cheung, Radford,
  and Chen]{salimans2016improved}
Tim Salimans, Ian Goodfellow, Wojciech Zaremba, Vicki Cheung, Alec Radford, and
  Xi~Chen.
\newblock Improved techniques for training gans.
\newblock \emph{arXiv preprint arXiv:1606.03498}, 2016.

\bibitem[Smith \& Topin(2019)Smith and Topin]{smith2019super}
Leslie~N Smith and Nicholay Topin.
\newblock Super-convergence: Very fast training of neural networks using large
  learning rates.
\newblock In \emph{Artificial Intelligence and Machine Learning for
  Multi-Domain Operations Applications}, volume 11006, pp.\  1100612.
  International Society for Optics and Photonics, 2019.

\bibitem[Smith et~al.(2021)Smith, Dherin, Barrett, and De]{smith2021origin}
Samuel~L Smith, Benoit Dherin, David~GT Barrett, and Soham De.
\newblock On the origin of implicit regularization in stochastic gradient
  descent.
\newblock \emph{arXiv preprint arXiv:2101.12176}, 2021.

\bibitem[Stich(2018)]{stich2018local}
Sebastian~U Stich.
\newblock Local sgd converges fast and communicates little.
\newblock \emph{arXiv preprint arXiv:1805.09767}, 2018.

\bibitem[Wang et~al.(2020)Wang, Tantia, Ballas, and Rabbat]{wang2020lookahead}
Jianyu Wang, Vinayak Tantia, Nicolas Ballas, and Michael Rabbat.
\newblock Lookahead converges to stationary points of smooth non-convex
  functions.
\newblock In \emph{ICASSP 2020-2020 IEEE International Conference on Acoustics,
  Speech and Signal Processing (ICASSP)}, pp.\  8604--8608. IEEE, 2020.

\bibitem[Woodworth et~al.(2020)Woodworth, Patel, Stich, Dai, Bullins, Mcmahan,
  Shamir, and Srebro]{woodworth2020local}
Blake Woodworth, Kumar~Kshitij Patel, Sebastian Stich, Zhen Dai, Brian Bullins,
  Brendan Mcmahan, Ohad Shamir, and Nathan Srebro.
\newblock Is local sgd better than minibatch sgd?
\newblock In \emph{International Conference on Machine Learning}, pp.\
  10334--10343. PMLR, 2020.

\bibitem[Yaida(2018)]{yaida2018fluctuation}
Sho Yaida.
\newblock Fluctuation-dissipation relations for stochastic gradient descent.
\newblock \emph{arXiv preprint arXiv:1810.00004}, 2018.

\bibitem[Zhang et~al.(2019)Zhang, Lucas, Hinton, and Ba]{zhang2019lookahead}
Michael~R Zhang, James Lucas, Geoffrey Hinton, and Jimmy Ba.
\newblock Lookahead optimizer: k steps forward, 1 step back.
\newblock \emph{arXiv preprint arXiv:1907.08610}, 2019.

\end{thebibliography}

\newpage
\appendix
\onecolumn
\section{MISSING PROOFS}
\subsection{Proofs for Section \ref{sec:convergence_to_stationary}} \label{proofs section 3}

Here, we prove the theorems about the convergence of 2-layers Lookahead to a stationary point, presented in \autoref{sec:convergence_to_stationary}.

This appendix is organized as follows: we start by two auxiliary technical lemmas, then prove Lemma \ref{main lemma} (Main lemma), which contains the main part of the proof for both \autoref{theorem:1} and \autoref{theorem:2}. Finally, we give a proof for the three announced theorems.

\subsubsection{Auxiliary Lemmas}

\begin{lemma}\label{lemma:1}
    Assumption (\ref{gamma:constraint_1}) implies the following two statements:
    \begin{gather}
        \nonumber 1) \frac{\sum_{r=0}^{R-1}\gamma_{r k_1 k_2}^2}{\sum_{r=0}^{R-1}\gamma_{r k_1 k_2}} \rightarrow 0, R \rightarrow \infty\\
        \nonumber 2) \frac{\sum_{r=0}^{R-1}\gamma_{r k_1 k_2}^3}{\sum_{r=0}^{R-1}\gamma_{r k_1 k_2}} \rightarrow 0, R \rightarrow \infty
    \end{gather}
\end{lemma}

\begin{proof}
Let us prove only the first statement, for the second one the proof is identical. Since the sequence $(\gamma_{r k_1 k_2})$ convergence to 0, we have:
\begin{gather}
    \nonumber \forall \epsilon>0 \: \exists S: \forall r\geq S \Rightarrow \gamma_{r k_1 k_2} < \epsilon.
\end{gather}
Therefore,
\begin{gather}
    \nonumber \frac{\sum_{r=0}^{R-1}\gamma_{r k_1 k_2}^2}{\sum_{r=0}^{R-1}\gamma_{r k_1 k_2}} = \frac{\sum_{r=0}^{S-1}\gamma_{r k_1 k_2}^2 + \sum_{r=S}^{R-1}\gamma_{r k_1 k_2}^2}{\sum_{r=0}^{R-1}\gamma_{r k_1 k_2}} \leq \frac{\sum_{r=0}^{S-1}\gamma_{r k_1 k_2}^2 + \epsilon \sum_{r=0}^{R-1}\gamma_{r k_1 k_2}}{\sum_{r=0}^{R-1}\gamma_{r k_1 k_2}}
\end{gather}

Now, since $\sum_{r=0}^{S-1}\gamma_{r k_1 k_2}^2$ is a finite sum, and $\sum_{r=0}^{R-1}\gamma_{r k_1 k_2} \rightarrow \infty$ we have:
\begin{gather}
    \nonumber \exists R_0: \forall R>R_0: \sum_{r=0}^{R-1}\gamma_{r k_1 k_2} > \frac{\sum_{r=0}^{S-1}\gamma_{r k_1 k_2}^2}{\epsilon} \Rightarrow \sum_{r=0}^{S-1}\gamma_{r k_1 k_2}^2 < \epsilon \sum_{r=0}^{R-1}\gamma_{r k_1 k_2}
\end{gather}
Substituting this in the previous inequality, we have:
\begin{gather}
    \nonumber \frac{\sum_{r=0}^{R-1}\gamma_{r k_1 k_2}^2}{\sum_{r=0}^{R-1}\gamma_{r k_1 k_2}} \leq \frac{\sum_{r=0}^{S-1}\gamma_{r k_1 k_2}^2 + \epsilon \sum_{r=0}^{R-1}\gamma_{r k_1 k_2}}{\sum_{r=0}^{R-1}\gamma_{r k_1 k_2}} < \frac{\epsilon \sum_{r=0}^{R-1}\gamma_{r k_1 k_2} + \epsilon \sum_{r=0}^{R-1}\gamma_{r k_1 k_2}}{\sum_{r=0}^{R-1}\gamma_{r k_1 k_2}} = 2 \epsilon \quad \forall T>R_0
\end{gather}
which completes the proof.

\end{proof}
 
\begin{lemma}\label{lemma}
Suppose that stochastic gradients $g(x_{j}, \xi_{j})$ satisfy assumptions (\ref{sg:independent}) - (\ref{sg:bounded}).
Then for any indexes $a \leq b$ we have:
$$\mathbb{E} \left[ \left\lVert\sum_{j=a}^{b} g(x_{j}, \xi_{j}) \right\rVert^2 \right] \leq 2 (b-a+1) \Big(\sigma^2 + \sum_{j=a}^{b} \mathbb{E}[\lVert\nabla f(x_{j})\rVert^2]\Big)$$
\end{lemma}

\begin{proof}

\begin{gather}
    \nonumber \mathbb{E} \left[ \left\lVert\sum_{j=a}^{b} g(x_{j}, \xi_{j}) \right\rVert^2 \right] = 
    \mathbb{E} \left[ \left\lVert \sum_{j=a}^{b} \Big(g(x_{j}, \xi_{j}) - \nabla f(x_{j})\Big) + \sum_{j=a}^{b} \nabla f(x_{j}) \right\rVert^2 \right] \leq \\
    \nonumber \leq 2 \left( \mathbb{E} \left[\left\lVert\sum_{j=a}^{b} \Big(g(x_{j}, \xi_{j}) - \nabla f(x_{j})\Big)\right\rVert^2 \right] + \mathbb{E} \left[ \left\lVert\sum_{j=a}^{b} \nabla f(x_{j})\right\rVert^2 \right] \right) = \\
    \nonumber = 2 \left(\sum_{j=a}^{b} \mathbb{E} \left[\lVert(g(x_{j}, \xi_{j}) - \nabla f(x_{j}))\rVert^2 \right] + \mathbb{E}\left[\left\lVert\sum_{j=a}^{b} \nabla f(x_{j})\right\rVert^2 \right]\right) \leq \\
    \nonumber \leq 2 \Big((b-a+1)\sigma^2 + (b-a+1)\sum_{j=a}^{b} \mathbb{E}[\lVert\nabla f(x_{j})\rVert^2]\Big) =\\
    \nonumber = 2 (b-a+1) \Big(\sigma^2 + \sum_{j=a}^{b} \mathbb{E}[\lVert\nabla f(x_{j})\rVert^2]\Big)
\end{gather}

In the first inequality, we used $\lVert p + q \rVert^2 \leq 2 (\lVert p \rVert^2 + \lVert q \rVert^2)$.

For the second equality, we used that for $j > i$:

\begin{gather*}
    \mathbb{E}  \left[ \langle g(x_{j}, \xi_{j}) - \nabla f(x_{j}), g(x_{i}, \xi_{i}) - \nabla f(x_{i})\rangle \right] = \mathbb{E}_{\xi_{i}} \left[ \mathbb{E}_{\xi_{j}} [\langle g(x_{j}, \xi_{j}) - \nabla f(x_{j}), g(x_{i}, \xi_{i}) - \nabla f(x_{i})\rangle | \xi_{i}] \right] = \\
    = \mathbb{E}_{\xi_{i}} \left[ \langle \mathbb{E}_{\xi_{j}} [g(x_{j}, \xi_{j}) - \nabla f(x_{j}) | \xi_{i}], g(x_{i}, \xi_{i}) - \nabla f(x_{i})\rangle \right] = 0
\end{gather*}

For the second inequality, we used assumption (\ref{sg:bounded}) and inequality $\lVert \sum_{i=1}^n p_i \rVert^2 \leq n \sum_{i=1}^n \lVert p_i \rVert^2$.

\end{proof}

\subsubsection{Main Lemma}

\begin{lemma}[Main lemma]\label{main lemma}

Suppose that learning rate is kept constant within each round: $\gamma_{r k_1 k_2 + i k_1 + j} = \gamma_{r k_1 k_2}$, $\forall r \geq 0$, $0 \leq i \leq k_2-1$, $0 \leq j \leq k_1-1$.
Then under the assumptions (\ref{smoothness}) - (\ref{gamma:constraint_2}) we have:

\begin{gather}
    \nonumber \frac{1}{k_1 k_2 S_R} \sum_{r=0}^{R-1} \Big(\gamma_{r k_1 k_2} \sum_{j=0}^{k_1 k_2-1} \mathbb{E}[\lVert\nabla f(\theta_{r k_1 k_2+j})\rVert^2]\Big) \leq
    \frac{2(f(\theta_0) - f_{inf})}{\alpha_1 \alpha_2 k_1 k_2 S_R} + \frac{\alpha_1 \alpha_2 L \sigma^2}{S_R} \sum_{r=0}^{R-1} \gamma_{r k_1 k_2}^2 + \\
    + \frac{2 L^2 \sigma^2 k_1}{S_R} \Big((1-\alpha_1)^2 \alpha_2^2 + 2(1-\alpha_2)^2 + \frac{4}{3} \alpha_1^2 (1-\alpha_2)^2 k_2^2 \Big) \sum_{r=0}^{R-1} \gamma_{r k_1 k_2}^3
    \label{main lemma formula}
\end{gather}

where $S_R = \sum_{r=0}^{R-1} \gamma_{r k_1 k_2}$ and $f_{inf}$ is the infimum of $f$.

\end{lemma}

\begin{proof}

Combining $L$-smoothness of $f$ with the formula (\ref{theta}), we have:

\begin{gather}
    \nonumber f(\theta_{t+1}) - f(\theta_t) \leq \langle \theta_{t+1} - \theta_t, \nabla f(\theta_t) \rangle + \frac{L}{2} \lVert\theta_{t+1} - \theta_t\rVert^2 = \\
    = - \gamma_t \alpha_1 \alpha_2 \langle \nabla f(\theta_t), g(x_t, \xi_t) \rangle + \frac{\gamma_t^2 \alpha_1^2 \alpha_2^2 L}{2} \lVert g(x_t, \xi_t) \rVert^2 \label{step:1}
\end{gather}

As in ~\cite{wang2020lookahead}, to simplify notations we denote by $\mathbb{E}_t[\cdot]$ the conditional expectation $\mathbb{E}_{\xi_t}[\cdot |\mathcal{F}_t]$, where $\mathcal{F}_t$ is the sigma algebra generated by the noise in stochastic gradients until iteration $t$. Taking the conditional expectation of both sides of (\ref{step:1}), for the first term on RHS \footnote{right-hand side} we have:
\begin{gather}
    \nonumber \mathbb{E}_{t} [\langle \nabla f(\theta_t), g(x_t, \xi_t) \rangle] = \langle \nabla f(\theta_t), \nabla f(x_t) \rangle = 
    \frac{1}{2}\Big(\lVert \nabla f(\theta_t)\rVert^2 + \lVert \nabla f(x_t)\rVert^2 - \lVert\nabla f(\theta_t) - \nabla f(x_t)\rVert^2\Big) \geq \\
    \nonumber \geq \frac{1}{2}\Big(\lVert \nabla f(\theta_t)\rVert^2 + \lVert\nabla f(x_t)\rVert^2 - L^2\lVert\theta_t - x_t\rVert^2\Big) = \\
    = \frac{1}{2} \Big(\lVert\nabla f(\theta_t)\rVert^2 + \lVert\nabla f(x_t)\rVert^2 - L^2\lVert(1 - \alpha_1 \alpha_2)x_t - (1 - \alpha_1)\alpha_2 y_t - (1 - \alpha_2) z_t\rVert^2\Big) \label{step:1 bound:1}
\end{gather}

For the second term on RHS of (\ref{step:1}):

\begin{gather}
    \nonumber \mathbb{E}_t[\lVert g(x_t, \xi_t)\rVert^2] = \mathbb{E}_t[\lVert g(x_t, \xi_t) - \nabla f(x_t) + \nabla f(x_t)\rVert^2] = \mathbb{E}_t[\lVert g(x_t, \xi_t) - \nabla f(x_t)\rVert^2] + \lVert\nabla f(x_t)\rVert^2 \leq \\
    \leq \sigma^2 + \lVert\nabla f(x_t)\rVert^2 \label{step:1 bound:2}
\end{gather}


Plugging (\ref{step:1 bound:1}) and (\ref{step:1 bound:2}) back into (\ref{step:1}) and taking the total expectation, we get

\begin{gather}
    \nonumber \mathbb{E}[f(\theta_{t+1})] - \mathbb{E}[f(\theta_t)] \leq - \frac{\gamma_t \alpha_1 \alpha_2}{2} \mathbb{E}[\lVert \nabla f(\theta_t)\rVert^2] - \frac{\gamma_t \alpha_1 \alpha_2}{2} (1 - \gamma_t \alpha_1 \alpha_2 L) \mathbb{E}[\lVert\nabla f(x_t)\rVert^2] + \frac{\gamma_t^2 \alpha_1^2 \alpha_2^2 L \sigma^2}{2} + \\
    + \frac{\gamma_t \alpha_1 \alpha_2 L^2}{2} \mathbb{E}[\lVert(1 - \alpha_1 \alpha_2)x_t - (1 - \alpha_1)\alpha_2 y_t - (1 - \alpha_2) z_t\rVert^2] \label{ineq:2}
\end{gather}

Let us decompose $t = r k_1 k_2 + i k_1 + j$, where $0 \leq i \leq k_2-1$, $0 \leq j \leq k_1-1$. For shortness, denote $s = r k_1 k_2 + i k_1 $, so that $t = s+j$. Summing (\ref{ineq:2}) from $j = 0$ to $j = k_1 - 1$, and using that $\gamma_{r k_1 k_2 + i k_1 + j} = \gamma_{r k_1 k_2}$, we get:

\begin{gather}
    \nonumber \mathbb{E}[f(\theta_{s+k_1})] - \mathbb{E}[f(\theta_s)] \leq - \frac{\gamma_{r k_1 k_2} \alpha_1 \alpha_2}{2} \sum_{j=0}^{k_1-1} \mathbb{E}[\lVert\nabla f(\theta_{s+j})\rVert^2] - \frac{\gamma_{r k_1 k_2} \alpha_1 \alpha_2}{2} (1 - \gamma_{r k_1 k_2} \alpha_1 \alpha_2 L) \sum_{j=0}^{k_1-1} \mathbb{E}[\lVert\nabla f(x_{s+j})\rVert^2] +\\
    + \frac{\gamma_{r k_1 k_2}^2 \alpha_1^2 \alpha_2^2 L \sigma^2 k_1}{2} 
    + \frac{\gamma_{r k_1 k_2} \alpha_1 \alpha_2 L^2}{2} \sum_{j=0}^{k_1-1} \mathbb{E}[\lVert(1 - \alpha_1 \alpha_2)x_{s+j} - (1 - \alpha_1)\alpha_2 y_{s+j} - (1 - \alpha_2) z_{s+j}\rVert^2] \label{three}
\end{gather}

Taking into account that $y_{s+j} = y_s$, $z_{s+j} = z_{r k_1 k_2}$, we can bound the last term on RHS of (\ref{three}) as follows:

\begin{gather}
    \nonumber \mathbb{E}[\lVert(1 - \alpha_1 \alpha_2)x_{s+j} - (1 - \alpha_1)\alpha_2 y_{s+j} - (1 - \alpha_2) z_{s+j}\rVert^2] = 
    \mathbb{E}[\lVert(1 - \alpha_1)\alpha_2 (x_{s+j} - y_{s}) + (1 - \alpha_2) (x_{s+j} - z_{r k_1 k_2})\rVert^2] \leq \\
    \nonumber \leq 2 \Big(\mathbb{E}[\lVert(1 - \alpha_1)\alpha_2 (x_{s+j} - y_{s})\rVert^2] + \mathbb{E}[\lVert(1 - \alpha_2) (x_{s+j} - z_{r k_1 k_2})\rVert^2]\Big) = \\
    = 2(1 - \alpha_1)^2\alpha_2^2 \mathbb{E}[\lVert x_{s+j} - y_{s}\rVert^2] + 2(1 - \alpha_2)^2\mathbb{E}[\lVert x_{s+j} - z_{r k_1 k_2}\rVert^2] \label{four}
\end{gather}

Let us bound the first term on RHS of (\ref{four}):

\begin{gather}
    \nonumber \mathbb{E}[\lVert x_{s+j} - y_{s}\rVert^2] = \mathbb{E}[\lVert x_{s+j} - x_{s}\rVert^2] = \gamma_{r k_1 k_2}^2\mathbb{E}\left[\left\lVert \sum_{j'=0}^{j-1} g(x_{s+j'}, \xi_{s+j'})\right\rVert^2 \right] \leq \text{using Lemma \ref{lemma}} \leq \\
    \leq 2 \gamma_{r k_1 k_2}^2 j \Big(\sigma^2 + \sum_{j'=0}^{j-1} \mathbb{E}[\lVert\nabla f(x_{s+j'})\rVert^2]\Big) \leq 2 \gamma_{r k_1 k_2}^2 j \Big(\sigma^2 + \sum_{j'=0}^{k_1-1} \mathbb{E}[\lVert\nabla f(x_{s+j'})\rVert^2]\Big)\label{two}
\end{gather}

Bounding the second term on RHS of (\ref{four}):

\begin{gather}
    \nonumber \mathbb{E}[\lVert x_{s+j} - z_{r k_1 k_2}\rVert^2] = \mathbb{E}[\lVert x_{r k_1 k_2+i k_1+j} - z_{r k_1 k_2}\rVert^2] = \mathbb{E}[\lVert x_{r k_1 k_2+i k_1+j} - y_{r k_1 k_2+i k_1} + y_{r k_1 k_2+i k_1} - z_{r k_1 k_2}\rVert^2] \leq \\
    \leq 2\Big(\mathbb{E}[\lVert x_{r k_1 k_2+i k_1+j} - y_{r k_1 k_2+i k_1}\rVert^2] + \mathbb{E}[\lVert y_{r k_1 k_2+i k_1} - z_{r k_1 k_2}\rVert^2]\Big) = 
    2\Big(\mathbb{E}[\lVert x_{s+j} - y_{s}\rVert^2] + \mathbb{E}[\lVert y_{r k_1 k_2+i k_1} - z_{r k_1 k_2}\rVert^2]\Big) \label{five}
\end{gather}

The first term can be bounded using (\ref{two}). Bounding the second term:

\begin{gather}
    \nonumber \mathbb{E}[\lVert y_{r k_1 k_2+i k_1} - z_{r k_1 k_2}\rVert^2] = \mathbb{E}[\lVert y_{r k_1 k_2+i k_1} - y_{r k_1 k_2}\rVert^2] = \\
    \nonumber = \mathbb{E} \left[\left\lVert\sum_{i'=0}^{i-1}(y_{r k_1 k_2+(i'+1)k_1}-y_{r k_1 k_2+i' k_1})\right\rVert^2\right] \leq i \sum_{i'=0}^{i-1} \mathbb{E} [\lVert y_{r k_1 k_2+(i'+1) k_1}-y_{r k_1 k_2+i' k_1}\rVert^2] = \\
    \nonumber = i \sum_{i'=0}^{i-1} \mathbb{E} \left[\left\lVert \alpha_1 \Big(y_{r k_1 k_2+i' k_1} - \gamma_{r k_1 k_2}\sum_{j=0}^{k_1-1} g(x_{r k_1 k_2+i' k_1+j}, \xi_{r k_1 k_2+i' k_1+j})\Big) + (1-\alpha_1)y_{r k_1 k_2+i' k_1} - y_{r k_1 k_2+i' k_1}\right\rVert^2\right] = \\
    \nonumber = i \alpha_1^2 \gamma_{r k_1 k_2}^2 \sum_{i'=0}^{i-1} \mathbb{E} \left[\left\lVert\sum_{j=0}^{k_1-1} g(x_{r k_1 k_2+i' k_1 + j}, \xi_{r k_1 k_2+i' k_1+j})\right\rVert^2\right] \leq \text{using Lemma \ref{lemma}} \leq \\
    \nonumber \leq i \alpha_1^2 \gamma_{r k_1 k_2}^2 \sum_{i'=0}^{i-1} 2 k_1 \Big(\sigma^2 + \sum_{j=0}^{k_1-1} \mathbb{E} [\lVert\nabla f(x_{r k_1 k_2+i' k_1+j})\rVert^2]\Big) = \\
    \nonumber = 2 i^2 \alpha_1^2 \gamma_{r k_1 k_2}^2 k_1 \sigma^2 + 2 i \alpha_1^2 \gamma_{r k_1 k_2}^2 k_1 \sum_{j=0}^{i k_1-1} \mathbb{E} [\lVert\nabla f(x_{r k_1 k_2+j})\rVert^2] \leq \\
    \leq 2 i^2 \alpha_1^2 \gamma_{r k_1 k_2}^2 k_1 \sigma^2 + 2 i \alpha_1^2 \gamma_{r k_1 k_2}^2 k_1 \sum_{j=0}^{k_1 k_2-1} \mathbb{E} [\lVert\nabla f(x_{r k_1 k_2+j})\rVert^2]
\end{gather}

Finally, the RHS of (\ref{five}) can be bounded by:

\begin{gather}
    \nonumber \mathbb{E}[\lVert x_{s+j} - z_{r k_1 k_2}\rVert^2] \leq \\
    \nonumber \leq 2\left(2 \gamma_{r k_1 k_2}^2 j \Big(\sigma^2 + \sum_{j'=0}^{k_1-1} \mathbb{E}[\lVert\nabla f(x_{s+j'})\rVert^2]\Big) + 2 i^2 \alpha_1^2 \gamma_{r k_1 k_2}^2 k_1 \sigma^2 + 2 i \alpha_1^2 \gamma_{r k_1 k_2}^2 k_1 \sum_{j'=0}^{k_1 k_2-1} \mathbb{E} [\lVert\nabla f(x_{r k_1 k_2+j'})\rVert^2]\right) = \\
    = 4 \gamma_{r k_1 k_2}^2 \left(j\Big(\sigma^2 + \sum_{j'=0}^{k_1-1} \mathbb{E}[\lVert\nabla f(x_{s+j'})\rVert^2]\Big) + i^2 \alpha_1^2 k_1 \sigma^2 + i \alpha_1^2 k_1 \sum_{j'=0}^{k_1 k_2-1} \mathbb{E} [\lVert\nabla f(x_{r k_1 k_2+j'})\rVert^2]\right) \label{six}
\end{gather}

Now, using (\ref{two}) and (\ref{six}), we can bound the RHS of (\ref{four}):

\begin{gather}
    \nonumber \mathbb{E}[\lVert(1 - \alpha_1 \alpha_2)x_{s+j} - (1 - \alpha_1)\alpha_2 y_{s+j} - (1 - \alpha_2) z_{s+j}\rVert^2] \leq \\
    \nonumber \leq 2(1 - \alpha_1)^2\alpha_2^2 \cdot 2 \gamma_{r k_1 k_2}^2 j \Big(\sigma^2 + \sum_{j'=0}^{k_1-1} \mathbb{E}[\lVert\nabla f(x_{s+j'})\rVert^2]\Big) + \\ 
    \nonumber + 2(1 - \alpha_2)^2 \cdot 4 \nonumber \gamma_{r k_1 k_2}^2 \left(j \Big(\sigma^2 + \sum_{j'=0}^{k_1-1} \mathbb{E}[\lVert\nabla f(x_{s+j'})\rVert^2] \Big) + 
    i^2 \alpha_1^2 k_1 \sigma^2 + i \alpha_1^2 k_1 \sum_{j'=0}^{k_1 k_2-1} \mathbb{E} [\lVert\nabla f(x_{r k_1 k_2+j'})\rVert^2] \right) = \\
    \nonumber = 4 \gamma_{r k_1 k_2}^2 \sigma^2 \Big((1 - \alpha_1)^2\alpha_2^2 j + 2(1 - \alpha_2)^2 j + 2(1 - \alpha_2)^2 \alpha_1^2 i^2 k_1 \Big) + \\
    \nonumber + 4 \gamma_{r k_1 k_2}^2 j \Big((1 - \alpha_1)^2\alpha_2^2 + 2(1 - \alpha_2)^2 \Big) \sum_{j'=0}^{k_1-1} \mathbb{E}[\lVert\nabla f(x_{s+j'})\rVert^2] + \\
    + 8 \gamma_{r k_1 k_2}^2 \alpha_1^2 (1-\alpha_2)^2 i k_1 \sum_{j'=0}^{k_1 k_2-1} \mathbb{E} [\lVert\nabla f(x_{r k_1 k_2+j'})\rVert^2]
\end{gather}

We have bounded the constituents of the last sum on RHS of (\ref{three}). Summing from $j=0$ to $j=k_1-1$, and using $\sum_{j=0}^{k_1-1} j < \frac{k_1^2}{2}$, we get:

\begin{gather}
    \nonumber \sum_{j=0}^{k_1-1} \mathbb{E}[\lVert(1 - \alpha_1 \alpha_2)x_{s+j} - (1 - \alpha_1)\alpha_2 y_{s+j} - (1 - \alpha_2) z_{s+j}\rVert^2] \leq \\
    \nonumber \leq 4 \gamma_{r k_1 k_2}^2 \sigma^2 \Big((1 - \alpha_1)^2\alpha_2^2 \frac{k_1^2}{2} + (1 - \alpha_2)^2 k_1^2 + 2(1 - \alpha_2)^2 \alpha_1^2 i^2 k_1^2 \Big) + \\
    \nonumber + 2 \gamma_{r k_1 k_2}^2 k_1^2 \Big((1 - \alpha_1)^2\alpha_2^2 + 2(1 - \alpha_2)^2 \Big) \sum_{j'=0}^{k_1-1} \mathbb{E}[\lVert\nabla f(x_{s+j'})\rVert^2] + \\
    + 8 \gamma_{r k_1 k_2}^2 \alpha_1^2 (1-\alpha_2)^2 i k_1^2 \sum_{j'=0}^{k_1 k_2-1} \mathbb{E} [\lVert\nabla f(x_{r k_1 k_2+j'})\rVert^2]
\end{gather}

Eventually, we can upper-bound the RHS of (\ref{three}):

\begin{gather}
    \nonumber \mathbb{E}[f(\theta_{s+k_1})] - \mathbb{E}[f(\theta_s)] \leq - \frac{\gamma_{r k_1 k_2} \alpha_1 \alpha_2}{2} \sum_{j=0}^{k_1-1} \mathbb{E}[\lVert\nabla f(\theta_{s+j})\rVert^2] - \\
    \nonumber - \frac{\gamma_{r k_1 k_2} \alpha_1 \alpha_2}{2} \Bigg(1 - \gamma_{r k_1 k_2} \alpha_1 \alpha_2 L - 2 \gamma_{r k_1 k_2}^2 L^2 k_1^2\Big((1-\alpha_1)^2 \alpha_2^2 + 2(1-\alpha_2)^2\Big)\Bigg) \sum_{j=0}^{k_1-1} \mathbb{E}[\lVert\nabla f(x_{s+j})\rVert^2] + \\
    \nonumber + \frac{\gamma_{r k_1 k_2}^2 \alpha_1^2 \alpha_2^2 L \sigma^2 k_1}{2}  + \gamma_{r k_1 k_2}^3 \alpha_1 \alpha_2 L^2 \sigma^2 k_1^2 \Big((1-\alpha_1)^2 \alpha_2^2 + 2(1-\alpha_2)^2 + 4\alpha_1^2(1-\alpha_2)^2 i^2\Big) +\\
    + 4 \gamma_{r k_1 k_2}^3 \alpha_1^3 \alpha_2 (1-\alpha_2)^2 L^2 i k_1^2 \sum_{j'=0}^{k_1 k_2-1} \mathbb{E} [\lVert\nabla f(x_{r k_1 k_2+j'})\rVert^2]
\end{gather}

Recall that $s = r k_1 k_2 + i k_1$. Summing from $i = 0$ to $i = k_2 -1$, and using $\sum_{i=0}^{k_2-1} i < \frac{k_2^2}{2}$, $\sum_{i=0}^{k_2-1} i^2 < \frac{k_2^3}{3}$, we get:

\begin{gather}
    \nonumber \mathbb{E}[f(\theta_{(r+1)k_1 k_2})] - \mathbb{E}[f(\theta_{r k_1 k_2})] \leq - \frac{\gamma_{r k_1 k_2} \alpha_1 \alpha_2}{2} \sum_{j'=0}^{k_1 k_2-1} \mathbb{E}[\lVert\nabla f(\theta_{r k_1 k_2+j'})\rVert^2] - \\
    \nonumber - \frac{\gamma_{r k_1 k_2} \alpha_1 \alpha_2}{2} \Bigg(1 - \gamma_{r k_1 k_2} \alpha_1 \alpha_2 L - 2 \gamma_{r k_1 k_2}^2 L^2 k_1^2\Big((1-\alpha_1)^2 \alpha_2^2 + 2(1-\alpha_2)^2\Big)\Bigg) \sum_{j'=0}^{k_1 k_2-1} \mathbb{E}[\lVert\nabla f(x_{r k_1 k_2+j'})\rVert^2] + \\
    \nonumber + \frac{\gamma_{r k_1 k_2}^2 \alpha_1^2 \alpha_2^2 L \sigma^2 k_1 k_2}{2} + \gamma_{r k_1 k_2}^3 \alpha_1 \alpha_2 L^2 \sigma^2 k_1^2 \Big((1-\alpha_1)^2 \alpha_2^2 k_2 + 2(1-\alpha_2)^2 k_2 + \frac{4}{3} \alpha_1^2 (1-\alpha_2)^2 k_2^3\Big) + \\
    \nonumber + 2 \gamma_{r k_1 k_2}^3 \alpha_1^3 \alpha_2 (1-\alpha_2)^2 L^2 k_1^2 k_2^2 \sum_{j'=0}^{k_1 k_2-1} \mathbb{E} [\lVert\nabla f(x_{r k_1 k_2+j'})\rVert^2] = \\
    \nonumber - \frac{\gamma_{r k_1 k_2} \alpha_1 \alpha_2}{2} \sum_{j'=0}^{k_1 k_2-1} \mathbb{E}[\lVert\nabla f(\theta_{r k_1 k_2+j'})\rVert^2] - \frac{\gamma_{r k_1 k_2} \alpha_1 \alpha_2}{2} \times \\
    \nonumber \times \Bigg(1 - \gamma_{r k_1 k_2} \alpha_1 \alpha_2 L - 2 \gamma_{r k_1 k_2}^2 L^2 k_1^2\Big((1-\alpha_1)^2 \alpha_2^2 + 2(1-\alpha_2)^2 + 2 \alpha_1^2 (1-\alpha_2)^2 k_2^2\Big) \Bigg) \sum_{j'=0}^{k_1 k_2-1} \mathbb{E}[\lVert\nabla f(x_{r k_1 k_2+j'})\rVert^2] + \\
    \nonumber + \frac{\gamma_{r k_1 k_2}^2 \alpha_1^2 \alpha_2^2 L \sigma^2 k_1 k_2}{2} + \gamma_{r k_1 k_2}^3 \alpha_1 \alpha_2 L^2 \sigma^2 k_1^2 \Big((1-\alpha_1)^2 \alpha_2^2 k_2 + 2(1-\alpha_2)^2 k_2 + \frac{4}{3} \alpha_1^2 (1-\alpha_2)^2 k_2^3 \Big)
\end{gather}

By assumption (\ref{gamma:constraint_2}), the term corresponding to $\sum_{j'=0}^{k_1 k_2-1} \mathbb{E}[\lVert\nabla f(x_{r k_1 k_2+j'})\rVert^2]$ is negative. 
Therefore, we can bound the RHS by excluding this term:

\begin{gather}
    \nonumber \mathbb{E}[f(\theta_{(r+1) k_1 k_2})] - \mathbb{E}[f(\theta_{r k_1 k_2})] \leq
    - \frac{\gamma_{r k_1 k_2} \alpha_1 \alpha_2}{2} \sum_{j'=0}^{k_1 k_2-1} \mathbb{E}[\lVert\nabla f(\theta_{r k_1 k_2+j'})\rVert^2] + \\
    + \frac{\gamma_{r k_1 k_2}^2 \alpha_1^2 \alpha_2^2 L \sigma^2 k_1 k_2}{2} + \gamma_{r k_1 k_2}^3 \alpha_1 \alpha_2 L^2 \sigma^2 k_1^2 \Big((1-\alpha_1)^2 \alpha_2^2 k_2 + 2(1-\alpha_2)^2 k_2 + \frac{4}{3} \alpha_1^2 (1-\alpha_2)^2 k_2^3\Big)
\end{gather}

Summing from $r = 0$ to $r = R-1$, we get:

\begin{gather}
    \nonumber \mathbb{E}[f(\theta_{R k_1 k_2})] - f(\theta_{0}) \leq - \frac{\alpha_1 \alpha_2}{2} \sum_{r=0}^{R-1} \Big(\gamma_{r k_1 k_2} \sum_{j'=0}^{k_1 k_2-1} \mathbb{E}[\lVert\nabla f(\theta_{r k_1 k_2+j'})\rVert^2] \Big) + \\
    \nonumber + \frac{\alpha_1^2 \alpha_2^2 L \sigma^2 k_1 k_2}{2} \sum_{r=0}^{R-1} \gamma_{r k_1 k_2}^2 + \\
    \nonumber + \alpha_1 \alpha_2 L^2 \sigma^2 k_2 k_1^2 \Big((1-\alpha_1)^2 \alpha_2^2 + 2(1-\alpha_2)^2 + \frac{4}{3} \alpha_1^2 (1-\alpha_2)^2 k_2^2 \Big) \sum_{r=0}^{R-1} \gamma_{r k_1 k_2}^3
\end{gather}

After rearranging and multiplying by $\frac{2}{\alpha_1 \alpha_2 k_1 k_2 S_R}$, where $S_R = \sum_{r=0}^{R-1} \gamma_{r k_1 k_2}$, and using that $\mathbb{E}[f(\theta_{R k_1 k_2})] \geq f_{inf}$, where $f_{inf}$ is the infimum of $f$, we obtain the statement of the lemma:

\begin{gather}
    \nonumber \frac{1}{k_1 k_2 S_R} \sum_{r=0}^{R-1} \Big(\gamma_{r k_1 k_2} \sum_{j'=0}^{k_1 k_2-1} \mathbb{E}[\lVert\nabla f(\theta_{r k_1 k_2+j'})\rVert^2]\Big) \leq
    \frac{2(f(\theta_0) - f_{inf})}{\alpha_1 \alpha_2 k_1 k_2 S_R} + \frac{\alpha_1 \alpha_2 L \sigma^2}{S_R} \sum_{r=0}^{R-1} \gamma_{r k_1 k_2}^2 + \\
    + \frac{2 L^2 \sigma^2 k_1}{S_R} \Big((1-\alpha_1)^2 \alpha_2^2 + 2(1-\alpha_2)^2 + \frac{4}{3} \alpha_1^2 (1-\alpha_2)^2 k_2^2 \Big) \sum_{r=0}^{R-1} \gamma_{r k_1 k_2}^3
\end{gather}

\end{proof}

\subsubsection{Proofs of the Stated Results}

Now, we have all the ingredients to prove the first two theorems. 

\begin{proof}[Proof of \autoref{theorem:1}]
Using assumption (\ref{gamma:constraint_1}) and Lemma (\ref{lemma:1}), we obtain that $1 / S_R \rightarrow 0$, $\sum_{r=0}^{R-1} \gamma_{r k_1 k_2}^2 / S_R \rightarrow 0$, $\sum_{r=0}^{R-1} \gamma_{r k_1 k_2}^3 / S_R \rightarrow 0$. Applying it to the statement of Lemma \ref{main lemma}, we get that RHS of (\ref{main lemma formula}) approaches to 0, which yields the result of the Theorem \ref{theorem:1}.
\end{proof}

\begin{proof}[Proof of \autoref{theorem:2}]
To prove Theorem \ref{theorem:2}, we substitute $\gamma_{r k_1 k_2} = \gamma = \text{const}$ in (\ref{main lemma formula}) and after minor simplifications obtain the desired statement.
\end{proof}

\begin{proof}[Proof of Claim \ref{conv_analysis: claim 1}]
When $\alpha_1 = \alpha_2 = 1$, the third term in the RHS of (\ref{stat conv theorem:2 formula}) disappears and the bound becomes:

\begin{gather*}
    \nonumber \frac{1}{T} \sum_{t=0}^{T-1} \mathbb{E}[\lVert\nabla f(\theta_{t})\rVert^2] \leq 
    \frac{2(f(\theta_0) - f_{inf})}{\gamma T} 
    + \gamma L \sigma^2
\end{gather*}

Optimizing it over $\gamma$, we get the optimal value $\gamma = \frac{1}{\sigma} \sqrt{\frac{2(f(\theta_0) - f_{inf})}{L T}}$ and the bound takes the form:

\begin{gather*}
    \nonumber \frac{1}{T} \sum_{t=0}^{T-1} \mathbb{E}[\lVert\nabla f(\theta_{t})\rVert^2]) \leq 
    \frac{2\sigma \sqrt{2 L (f(\theta_0) - f_{inf})}}{\sqrt{T}}
\end{gather*}

Note, that we need $T \geq \frac{2(f(\theta_0) - f_{inf})}{L \sigma^2 \gamma_*^2}$ to ensure the constraint (\ref{gamma:constraint_2}).

For any other choice of $\alpha_1$ and $\alpha_2$, the third term in the RHS of (\ref{stat conv theorem:2 formula}) is positive. Thus,

\begin{gather*}
    \frac{2(f(\theta_0) - f_{inf})}{\gamma \alpha_1 \alpha_2 T} + \gamma \alpha_1 \alpha_2 L \sigma^2 
    + 2 \gamma^2 L^2 \sigma^2 k_1 \left( (1-\alpha_1)^2 \alpha_2^2 + 2(1-\alpha_2)^2 + \frac{4}{3} \alpha_1^2 (1-\alpha_2)^2 k_2^2 \right) > \\
    > \frac{2(f(\theta_0) - f_{inf})}{\gamma \alpha_1 \alpha_2 T} + \gamma \alpha_1 \alpha_2 L \sigma^2 \geq |\text{by Cauchy's inequality}| \geq \frac{2\sigma \sqrt{2 L (f(\theta_0) - f_{inf})}}{\sqrt{T}}
\end{gather*}

which completes the proof.
\end{proof}

\begin{proof}[Proof of Corollary \ref{corollary:2}]
For any choice of $\gamma$, we get:

\begin{gather*}
    \frac{2(f(\theta_0) - f_{inf})}{\gamma \alpha_1 \alpha_2 T} + \gamma \alpha_1 \alpha_2 L \sigma^2
    + 2 \gamma^2 L^2 \sigma^2 k_1 \left( (1-\alpha_1)^2 \alpha_2^2 + 2(1-\alpha_2)^2 + \frac{4}{3} \alpha_1^2 (1-\alpha_2)^2 k_2^2 \right) \geq \\
    \geq \frac{2(f(\theta_0) - f_{inf})}{\gamma \alpha_1 \alpha_2 T} + \gamma \alpha_1 \alpha_2 L \sigma^2 \geq |\text{by Cauchy's inequality}| \geq \frac{2\sigma \sqrt{2 L (f(\theta_0) - f_{inf})}}{\sqrt{T}}
\end{gather*}

so the constant before $\frac{1}{\sqrt{T}}$ indeed cannot be decreased. By substituting the proposed value of $\gamma$, which was found as the equality point of Cauchy's inequality, we get the declared result.
\end{proof}

\begin{theorem} \label{theorem:3}
Let us call the \emph{run} to be an interval between two consecutive restarts. Consider 2-layers Lookahead with restarts defined as follows. For each $m \geq 0$, we start the run $m$ from the fixed initial point $z_0$ and perform $R_m = 4^m$ rounds of 2-layers Lookahead using learning rate $\gamma_m = \frac{\gamma_*}{\sqrt{R_m}} = \frac{\gamma_*}{2^m}$. As usual, denote by $R$ the total number of rounds and by  $T = R k_1 k_2$ the total number of iterations. Then, we have:

\begin{gather}
    \nonumber \frac{1}{T} \sum_{t=0}^{T-1} \mathbb{E}[\lVert\nabla f(\theta_{t})\rVert^2] = \mathcal{O}\left(\frac{1}{\sqrt{T}}\right), T \rightarrow +\infty
\end{gather}

\end{theorem}

\begin{proof}
Let fix the number of rounds $R$ and consider a number of run $M$ at which round $R$ occurs. Such $M$ satisfies $\sum_{m=0}^{M-1} 4^m < R \leq \sum_{m=0}^{M} 4^m$. Denote

\begin{gather*}
    S_m = \sum_{t'=0}^{R_m k_1 k_2-1} \mathbb{E}[\lVert\nabla f(\theta_{t'})\rVert^2]
\end{gather*}

where $\{\theta_{t'}\}_{t'=0}^{R_m k_1 k_2-1}$ - sequence produced by 2-layers Lookahead, starting from the point $z_0$ and using learning rate $\gamma_m$. Then $\{\theta_{t'}\}_{t'=0}^{R_m k_1 k_2-1}$ coincides with the sequence produced by the $m$-th run of 2-layers Lookahead with restarts. By \autoref{theorem:2}, we have:

\begin{gather*}
    \frac{1}{k_1 k_2 R_m} S_m \leq \frac{b_0}{\sqrt{k_1 k_2 R_m}} + \frac{b_1}{k_1 k_2 R_m} = \frac{b_0}{2^m \sqrt{k_1 k_2}} + \frac{b_1}{4^m k_1 k_2}
\end{gather*}

where $b_0$ and $b_1$ - some constants, defined from (\ref{corrolary:1 formula}). Multiplying by $k_1 k_2 R_m = k_1 k_2 4^m$, we get:

\begin{gather*}
    S_m \leq c_0 2^m + c_1
\end{gather*}

where $c_0 = b_0 \sqrt{k_1 k_2}$, $c_1 = b_1$.

Now, using that $R\leq \sum_{m=0}^{M} 4^m =: B$, we can estimate:

\begin{gather*}
    \sum_{t=0}^{T - 1} \mathbb{E}[\lVert\nabla f(\theta_{t})\rVert^2] =
    \sum_{t=0}^{R k_1 k_2 - 1} \mathbb{E}[\lVert\nabla f(\theta_{t})\rVert^2] \leq 
    \sum_{t=0}^{B k_1 k_2 - 1} \mathbb{E}[\lVert\nabla f(\theta_{t})\rVert^2] = 
    \sum_{m=0}^{M} S_m \leq \\
    \leq \sum_{m=0}^{M} (c_0 2^m + c_1) = 
    c_0 (2^{M+1}-1) + c_1 (M+1)
\end{gather*}

Further, we have $R \geq 1 + \sum_{m=0}^{M-1} 4^m = 1 + \frac{4^{M}-1}{3} > \frac{4^{M}}{3}$. Thus, $12 R > 4^{M+1} \Rightarrow 2^{M+1} < 2 \sqrt{3R}$. Also, $R > \frac{4^{M}}{3} > 4^{M-1} \Rightarrow M-1 < \log_4(R)$. Substituting it in the last expression, we get:

\begin{gather*}
    c_0 (2^{M+1}-1) + c_1 (M+1) < 2 c_0 \sqrt{3R} + c_1 (\log_4(R) + 2)
\end{gather*}

Finally, we obtain:

\begin{gather*}
    \sum_{t=0}^{R k_1 k_2 - 1} \mathbb{E}[\lVert\nabla f(\theta_{t})\rVert^2] < 2 c_0 \sqrt{3R} + c_1 (\log_4(R) + 2)
\end{gather*}

Dividing back by normalization constant, we get:

\begin{gather*}
    \frac{1}{T} \sum_{t=0}^{T-1} \mathbb{E}[\lVert\nabla f(\theta_{t})\rVert^2] =
    \frac{1}{R k_1 k_2} \sum_{t=0}^{R k_1 k_2 - 1} \mathbb{E}[\lVert\nabla f(\theta_{t})\rVert^2] < 
    \frac{2 c_0 \sqrt{3R}}{R k_1 k_2} + \frac{c_1 (\log_4(R) + 2)}{R k_1 k_2} = 
    \frac{2 c_0 \sqrt{3}}{\sqrt{R} k_1 k_2} + \frac{c_1 (\log_4(R) + 2)}{R k_1 k_2} = \\
    = \mathcal{O}(\frac{1}{\sqrt{R}}) =
    \mathcal{O}(\frac{1}{\sqrt{T}})
\end{gather*}

\end{proof}

\subsection{Proof of Theorem \ref{theorem:impl reg LA}} \label{proof theorem 4}

Denote by $LA^{(n)}(y_0, \{r, k_n, \ldots, k_1\}, \{\alpha_n, \ldots, \alpha_1\})$ the output of the $n$-layers Lookahead with parameters $k = (k_1, \ldots, k_n)$, $\alpha = (\alpha_1, \ldots, \alpha_n)$, which starts from the point $y_0$ and performs $r$ rounds using SGD as the base optimizer.

\begin{lemma}
For $y_r = LA^{(n)}(y_0, \{r, k_n, \ldots, k_1\}, \{\alpha_n, \ldots, \alpha_1\})$, we have:

\begin{gather}
    y_r - y_0 = -\gamma \alpha_{n} \ldots \alpha_1 k_{n} \ldots k_1 \sum_{i=0}^{r-1} \nabla f_i^{(k_{n} \ldots k_1)} (y_0) + \mathcal{O}(\gamma^2) \label{generalization: 24}
\end{gather}
\end{lemma}

\begin{proof}
Base: $n = 0$ ($LA^{(n)}$ degenerates to SGD):

\begin{gather*}
    y_r - y_0 = \sum_{i=0}^{r-1} (y_{i+1} - y_i) = - \gamma \sum_{i=0}^{r-1} \nabla f_i(y_i) = - \gamma \sum_{i=0}^{r-1} \nabla f_i(y_0) + \mathcal{O}(\gamma^2)
\end{gather*}

Induction step: suppose the statement holds for $LA^{(n)}$, Let us prove it for $y_r = LA^{(n+1)}(y_0, \{r, k_{n+1}, \ldots, k_1\}, \{\alpha_{n+1}, \ldots, \alpha_1\})$.

\begin{gather}
    \nonumber y_r - y_0 =
    \sum_{i=0}^{r-1} (y_{i+1}-y_i) =
    \sum_{i=0}^{r-1} \Big( (1-\alpha_{n+1}) y_i + \alpha_{n+1} x_{i, k_{n+1}} - y_i \Big) =
    \sum_{i=0}^{r-1} \alpha_{n+1} (x_{i,k_{n+1}} - y_i) = \\
    = \alpha_{n+1} \sum_{i=0}^{r-1} (x_{i, k_{n+1}} - x_{i,0})
\end{gather}

Now, using the induction hypothesis for $x_{i,k_{n+1}} = LA^{(n)}(x_{i,0}, \{k_{n+1}, k_n, \ldots, k_1\}, \{\alpha_n, \ldots, \alpha_1\})$, we can continue:

\begin{gather}
    \nonumber y_r - y_0 = 
    \alpha_{n+1} \sum_{i=0}^{r-1} (-\gamma \alpha_{n} \ldots \alpha_1 k_{n} \ldots k_1 \sum_{j=0}^{k_{n+1}-1} \nabla f_{i k_{n+1}+j}^{(k_{n} \ldots k_1)} (x_{i,0}) + \mathcal{O}(\gamma^2)) = \\
    \nonumber = - \gamma \alpha_{n+1} \ldots \alpha_1 k_{n} \ldots k_1 \sum_{i=0}^{r-1} \sum_{j=0}^{k_{n+1}-1} \nabla f_{i k_{n+1}+j}^{(k_{n} \ldots k_1)} (y_0) + \mathcal{O}(\gamma^2) = \\
    = - \gamma \alpha_{n+1} \ldots \alpha_1 k_{n+1} \ldots k_1 \sum_{i=0}^{r-1} \nabla f_{i}^{(k_{n+1} \ldots k_1)} (y_0) + \mathcal{O}(\gamma^2)
\end{gather}

which finishes the proof.

\end{proof}

\begin{lemma}

For $y_r = LA^{(n)}(y_0, \{r, k_n, \ldots, k_1\}, \{\alpha_n, \ldots, \alpha_1\})$, we have:

\begin{gather}
    \nonumber \mathbb{E}[y_r] = y_0 
    - \gamma \alpha_n \ldots \alpha_1 r k_n \ldots k_1 \nabla f(y_0)
    + \frac{\gamma^2 \alpha_n^2 \ldots \alpha_1^2 r^2 k_n^2 \ldots k_1^2}{4} \nabla \| \nabla f(y_0) \|^2 +\\
    \nonumber + \frac{\gamma^2 \alpha_n (1-\alpha_n) \alpha_{n-1}^2 \ldots \alpha_1^2 k_n^2 \ldots k_1^2}{4} \mathbb{E} \left[ \sum_{i=0}^{r-1} \nabla \| \nabla f_i^{(k_n \ldots k_1)} (y_0) \|^2 \right] + \\
    \nonumber +  \frac{\gamma^2 \alpha_n \alpha_{n-1} (1-\alpha_{n-1}) \alpha_{n-2}^2 \ldots \alpha_1^2 k_{n-1}^2 \ldots k_1^2}{4} \mathbb{E} \left[ \sum_{i=0}^{rk_n-1} \nabla \| \nabla f_i^{(k_{n-1}\ldots k_1)} \|^2 \right] + \\
    \nonumber + \ldots + \\
    \nonumber + \frac{\gamma^2 \alpha_n \ldots \alpha_2 \alpha_1 (1-\alpha_1)k_1^2}{4} \mathbb{E} \left[ \sum_{i=0}^{r k_n \ldots k_2-1} \nabla \|\nabla f_i^{(k_1)}(y_0)\|^2 \right] - \\
    - \frac{\gamma^2 \alpha_n \ldots \alpha_1}{4} \mathbb{E} \left[ \sum_{i=0}^{r k_n \ldots k_1-1} \nabla \|\nabla f_i(y_0)\|^2 \right]
    + \mathcal{O}(\gamma^3) \label{generalization: 31}
\end{gather}

\end{lemma}

\begin{proof}

Let us prove by induction by the number of layers $n$. 

Base ($n = 0$). For the corresponding result for SGD, see formula ($18$) in \citet{smith2021origin}.

Induction step ($n \Rightarrow n+1$). 

\begin{gather}
    \nonumber y_r - y_0 =
    \sum_{i=0}^{r-1} (y_{i+1}-y_i) =
    \sum_{i=0}^{r-1} \Big( (1-\alpha_{n+1}) y_i + \alpha_{n+1} x_{i, k_{n+1}} - y_i \Big) =
    \sum_{i=0}^{r-1} \alpha_{n+1} (x_{i,k_{n+1}} - y_i) = \\
    = \alpha_{n+1} \sum_{i=0}^{r-1} (x_{i, k_{n+1}} - x_{i,0})
\end{gather}

Taking the expectation of both sides, we get:

\begin{gather}
    \mathbb{E}[y_r] - y_0 = \alpha_{n+1} \sum_{i=0}^{r-1} (\mathbb{E}[x_{i, k_{n+1}}] - \mathbb{E}[x_{i,0}])
\end{gather}

Recall, one epoch of $LA^{(n+1)}$ contains $r$ rounds, and thus $r$ calls of $LA^{(n)}$ for updating its inner variable $x$. For each round $i$, denote by $\mathbb{E}_n$ the expectation w.t.r. the randomness inside of this round. In other words, we consider the input $x_{i,0}$ and the set of samples used during $i$-th round of $LA^{(n+1)}$ to be fixed, and the expectation is taken across all possible shuffles of these samples across the mini-bathes $\{f_{i k_n \ldots k_1 + j}\}_{j=0}^{k_n \ldots k_1-1}$.

By the induction hypothesis, applied to $x_{i, k_{n+1}} = LA^{(n)}(x_{i,0}, \{k_{n+1}, \ldots, k_1\}, \{\alpha_n, \ldots, \alpha_1\})$, for every $i \in \{0, \ldots, r-1\}$ we get:

\begin{gather}
    \nonumber \mathbb{E}_n[x_{i, k_{n+1}}] - x_{i,0} =
    - \gamma \alpha_n \ldots \alpha_1 k_{n+1} \ldots k_1 \nabla f_i^{(k_{n+1}\ldots k_1)}(x_{i,0}) +\\
    \nonumber + \frac{\gamma^2 \alpha_n^2 \ldots \alpha_1^2 k_{n+1}^2 \ldots k_1^2}{4} \nabla \|\nabla f_i^{(k_{n+1} \ldots k_1)}(x_{i,0}) \|^2 +\\
    \nonumber + \frac{\gamma^2 \alpha_n(1-\alpha_n) \alpha_{n-1}^2 \ldots \alpha_1^2 k_n^2 \ldots k_1^2}{4} \mathbb{E}_n \left[ \sum_{j=0}^{k_{n+1}-1} \nabla \| \nabla f_{i k_{n+1}+j}^{(k_n \ldots k_1)}(x_{i,0}) \|^2 \right] + \\
    \nonumber + \frac{\gamma^2 \alpha_n \alpha_{n-1} (1-\alpha_{n-1}) \alpha_{n-2}^2 \ldots \alpha_1^2 k_{n-1}^2 \ldots k_1^2}{4} \mathbb{E}_n \left[ \sum_{j=0}^{k_{n+1}k_n-1} \nabla \|\nabla f_{i k_{n+1} k_n +j}^{(k_{n-1}\ldots k_1)}(x_{i,0)} \|^2 \right] + \\
    \nonumber + \ldots + \\
    \nonumber + \frac{\gamma^2 \alpha_n \ldots \alpha_2 \alpha_1 (1-\alpha_1) k_1^2}{4} \mathbb{E}_n \left[ \sum_{j=0}^{k_{n+1}\ldots k_2-1} \nabla \|\nabla f_{i k_{n+1} \ldots k_2 +j}^{(k_1)}(x_{i,0})  \|^2 \right] -\\
    - \frac{\gamma^2 \alpha_n \ldots \alpha_1}{4} \mathbb{E}_n \left[ \sum_{j=0}^{k_{n+1} \ldots k_1-1} \nabla \| \nabla f_{i k_{n+1} \ldots k_1 +j}(x_{i,0}) \|^2 \right]
    + \mathcal{O}(\gamma^3)
\end{gather}

Summing up for $i = 0 \ldots r-1$, multiplying by $\alpha_{n+1}$, taking the total expectation, and using the law of total expectation ($\mathbb{E}[\mathbb{E}_n[\ldots]] = \mathbb{E}[\ldots]$), we get:

\begin{gather}
    \nonumber \mathbb{E}[y_r] - y_0 =
    - \gamma \alpha_{n+1} \ldots \alpha_1 k_{n+1} \ldots k_1 \mathbb{E} \left[ \sum_{i=0}^{r-1} \nabla f_i^{(k_{n+1} \ldots k_1)}(x_{i,0}) \right] + \\
    \nonumber + \frac{\gamma^2 \alpha_{n+1} \alpha_n^2 \ldots \alpha_1^2 k_{n+1}^2 \ldots k_1^2}{4} \mathbb{E} \left[ \sum_{i=0}^{r-1} \nabla \|\nabla f_i^{(k_{n+1} \ldots k_1)}(x_{i,0}) \|^2 \right]+\\
    \nonumber + \frac{\gamma^2 \alpha_{n+1} \alpha_n(1-\alpha_n) \alpha_{n-1}^2 \ldots \alpha_1^2 k_n^2 \ldots k_1^2}{4} \mathbb{E} \left[ \sum_{i=0}^{r-1} \sum_{j=0}^{k_{n+1}-1} \nabla \| \nabla f_{i k_{n+1}+j}^{(k_n \ldots k_1)}(x_{i,0}) \|^2 \right] + \\
    \nonumber + \ldots + \\
    \nonumber + \frac{\gamma^2 \alpha_{n+1} \ldots \alpha_2 \alpha_1 (1-\alpha_1) k_1^2}{4} \mathbb{E} \left[ \sum_{i=0}^{r-1} \sum_{j=0}^{k_{n+1}\ldots k_2-1} \nabla \|\nabla f_{i k_{n+1} \ldots k_2 +j}^{(k_1)}(x_{i,0})  \|^2 \right] -\\
    - \frac{\gamma^2 \alpha_{n+1} \ldots \alpha_1}{4} \mathbb{E} \left[ \sum_{i=0}^{r-1} \sum_{j=0}^{k_{n+1} \ldots k_1-1} \nabla \| \nabla f_{i k_{n+1} \ldots k_1 +j}(x_{i,0}) \|^2 \right]
    + \mathcal{O}(\gamma^3)
\end{gather}

Combining the sums over $i$ and $j$, and using that $x_{i,0} = y_0 + \mathcal{O}(\gamma)$, so that we can replace the function argument $x_{i,0}$ by $y_0$ in all second-order terms, we can rewrite the last formula in the following way: 

\begin{gather}
    \nonumber \mathbb{E}[y_r] - y_0 =
    - \gamma \alpha_{n+1} \ldots \alpha_1 k_{n+1} \ldots k_1 \mathbb{E} \left[ \sum_{i=0}^{r-1} \nabla f_i^{(k_{n+1} \ldots k_1)}(x_{i,0}) \right] + \\
    \nonumber + \frac{\gamma^2 \alpha_{n+1} \alpha_n^2 \ldots \alpha_1^2 k_{n+1}^2 \ldots k_1^2}{4} \mathbb{E} \left[ \sum_{i=0}^{r-1} \nabla \|\nabla f_i^{(k_{n+1} \ldots k_1)}(y_0) \|^2 \right] +\\
    \nonumber + \frac{\gamma^2 \alpha_{n+1} \alpha_n(1-\alpha_n) \alpha_{n-1}^2 \ldots \alpha_1^2 k_n^2 \ldots k_1^2}{4} \mathbb{E} \left[ \sum_{i=0}^{r k_{n+1}-1} \nabla \| \nabla f_{i}^{(k_n \ldots k_1)}(y_0) \|^2 \right] + \\
    \nonumber + \ldots + \\
    \nonumber + \frac{\gamma^2 \alpha_{n+1} \ldots \alpha_2 \alpha_1 (1-\alpha_1) k_1^2}{4} \mathbb{E} \left[ \sum_{i=0}^{r k_{n+1}\ldots k_2-1} \nabla \|\nabla f_{i}^{(k_1)}(y_0)  \|^2 \right] -\\
    - \frac{\gamma^2 \alpha_{n+1} \ldots \alpha_1}{4} \mathbb{E} \left[ \sum_{i=0}^{r k_{n+1} \ldots k_1-1} \nabla \| \nabla f_{i}(y_0) \|^2 \right]
    + \mathcal{O}(\gamma^3) \label{generalization: 30}
\end{gather}

Now let us transform the term $\mathbb{E} \left[ \sum_{i=0}^{r-1} \nabla f_i^{(k_{n+1} \ldots k_1)}(x_{i,0}) \right]$ to get rid of the dependency on $x_{i,0}$.

\begin{gather}
    \nabla f_i^{(k_{n+1} \ldots k_1)}(x_{i,0}) =
    \nabla f_i^{(k_{n+1} \ldots k_1)}(y_i) =
    \nabla f_i^{(k_{n+1} \ldots k_1)}(y_0)
    + \nabla^2 f_i^{(k_{n+1} \ldots k_1)}(y_0) (y_i - y_0)
    + \mathcal{O}(\gamma^2)
\end{gather}

Substituting $r=i$, $n=n+1$ in (\ref{generalization: 24}) to express $y_i - y_0$, we can continue:

\begin{gather}
    \nonumber \nabla f_i^{(k_{n+1} \ldots k_1)}(x_{i,0}) =
    \nabla f_i^{(k_{n+1} \ldots k_1)}(y_0) - \\
    - \gamma \alpha_{n+1} \ldots \alpha_1 k_{n+1} \ldots k_1 \sum_{j=0}^{i-1} \nabla^2 f_i^{(k_{n+1} \ldots k_1)}(y_0) \nabla f_j^{(k_{n+1} \ldots k_1)}(y_0)
    + \mathcal{O}(\gamma^2) \label{generalization: 25}
\end{gather}

Summing (\ref{generalization: 25}) for $i=0 \ldots r-1$ and taking the expectation, we get:

\begin{gather}
    \nonumber \mathbb{E} \left[ \sum_{i=0}^{r-1} \nabla f_i^{(k_{n+1} \ldots k_1)}(x_{i,0}) \right] =
    \mathbb{E} \left[ \sum_{i=0}^{r-1} \nabla f_i^{(k_{n+1} \ldots k_1)}(y_0) \right] - \\
    - \gamma \alpha_{n+1} \ldots \alpha_1 k_{n+1} \ldots k_1 \mathbb{E} \left[ \sum_{i=0}^{r-1} \sum_{j=0}^{i-1} \nabla^2 f_i^{(k_{n+1} \ldots k_1)}(y_0) \nabla f_j^{(k_{n+1} \ldots k_1)}(y_0) \right]
    + \mathcal{O}(\gamma^2) \label{generalization: 26}
\end{gather}

Let us simplify the right-hand side.

\begin{gather}
    \mathbb{E} \left[ \sum_{i=0}^{r-1} \nabla f_i^{(k_{n+1} \ldots k_1)}(y_0) \right] =
    \mathbb{E} \left[ r \nabla f(y_0) \right] =
    r \nabla f(y_0) \label{generalization: 27}
\end{gather}

\begin{gather}
    \nonumber \mathbb{E} \left[ \sum_{i=0}^{r-1} \sum_{j=0}^{i-1} \nabla^2 f_i^{(k_{n+1} \ldots k_1)}(y_0) \nabla f_j^{(k_{n+1} \ldots k_1)}(y_0) \right] =
    \frac{1}{2} \mathbb{E} \left[ \sum_{0 \leq i\neq j \leq r-1} \nabla^2 f_i^{(k_{n+1} \ldots k_1)}(y_0) \nabla f_j^{(k_{n+1} \ldots k_1)}(y_0) \right] = \\
    \nonumber = \frac{1}{2} \mathbb{E} \left[ \sum_{i=0}^{r-1} \nabla^2 f_i^{(k_{n+1} \ldots k_1)}(y_0) \sum_{j=0}^{r-1} \nabla f_j^{(k_{n+1} \ldots k_1)}(y_0) - \sum_{i=0}^{r-1} \nabla^2 f_i^{(k_{n+1} \ldots k_1)}(y_0) \nabla f_i^{(k_{n+1} \ldots k_1)}(y_0)\right] = \\
    \nonumber = \frac{1}{2} \mathbb{E} \left[ r^2 \nabla^2 f(y_0) \nabla f(y_0) - \sum_{i=0}^{r-1} \nabla^2 f_i^{(k_{n+1} \ldots k_1)}(y_0) \nabla f_i^{(k_{n+1} \ldots k_1)}(y_0) \right] = \\
    = \frac{r^2}{4} \nabla \| \nabla f(y_0) \|^2 - \frac{1}{4} \mathbb{E} \left[ \sum_{i=0}^{r-1} \nabla \|\nabla f_i^{(k_{n+1} \ldots k_1)}(y_0)\|^2 \right]
    \label{generalization: 28}
\end{gather}

Substituting (\ref{generalization: 27}) and (\ref{generalization: 28}) in (\ref{generalization: 26}), we obtain:

\begin{gather}
    \nonumber \mathbb{E} \left[ \sum_{i=0}^{r-1} \nabla f_i^{(k_{n+1} \ldots k_1)}(x_{i,0}) \right] =
    r \nabla f(y_0)
    - \frac{\gamma \alpha_{n+1} \ldots \alpha_1 r^2 k_{n+1} \ldots k_1}{4} \nabla \| \nabla f(y_0) \|^2 + \\
    + \frac{\gamma \alpha_{n+1} \ldots \alpha_1 k_{n+1} \ldots k_1}{4} \mathbb{E} \left[ \sum_{i=0}^{r-1} \nabla \|\nabla f_i^{(k_{n+1} \ldots k_1)}(y_0)\|^2 \right]
    \label{generalization: 29}
\end{gather}

Finally, substituting (\ref{generalization: 29}) in (\ref{generalization: 30}) and combining similar terms, we get the desired formula for $\mathbb{E}[y_r] - y_0$ for $(n+1)$-layers Lookahead.

\end{proof}

\begin{proof}[Proof of \autoref{theorem:impl reg LA}]

As before, let $m = r k_n \ldots k_1$, where $r$ - number of rounds of $n$-layers Lookahead per epoch, and denote $\beta = \alpha_n \ldots \alpha_1$. 
Then, we can rewrite the formula (\ref{generalization: 31}) in more compact way:

\begin{gather}
    \nonumber \mathbb{E}[y_r] = y_0
    - \gamma \beta m \nabla f(y_0)
    + \frac{\gamma^2 \beta^2 m^2}{4} \nabla \| \nabla f(y_0) \|^2 + \\
    \nonumber + \sum_{p=1}^{n} \frac{\gamma^2 \alpha_n \ldots \alpha_p (1 - \alpha_p) \alpha_{p-1}^2 \ldots \alpha_1^2 k_p^2 \ldots k_1^2}{4}
    \mathbb{E} \left[ \sum_{i=0}^{r k_n \ldots k_{p+1}-1} \nabla \| \nabla f_i^{(k_p \ldots k_1)} (y_0) \|^2 \right] - \\
    - \frac{\gamma^2 \beta}{4} \mathbb{E} \left[ \sum_{i=0}^{m-1} \nabla \| \nabla f_i(y_0) \|^2 \right]
    + \mathcal{O}(\gamma^3) \label{generalization: 32}
\end{gather}

Let us express $\mathbb{E} \left[ \sum_{i=0}^{r k_n \ldots k_{p+1}-1} \nabla \| \nabla f_i^{(k_p \ldots k_1)} (y_0) \|^2 \right]$ in terms of $ANG(y_0)$ and $AIG(y_0)$.

\begin{gather}
    \nonumber \mathbb{E} \left[ \sum_{i=0}^{r k_n \ldots k_{p+1}-1} \nabla \| \nabla f_i^{(k_p \ldots k_1)} (y_0) \|^2 \right] =
    r k_n \ldots k_{p+1} \mathbb{E} \left[ \nabla \| \nabla f_0^{(k_p \ldots k_1)} (y_0) \|^2 \right] = \\
    \nonumber = r k_n \ldots k_{p+1}\frac{1}{k_p^2 \ldots k_1^2} \mathbb{E} \left[ \nabla \| \sum_{j=0}^{k_p \ldots k_1-1} \nabla f_j(y_0) \|^2 \right] = \\
    \nonumber = r k_n \ldots k_{p+1}\frac{1}{k_p^2 \ldots k_1^2} 
    \mathbb{E} \left[ \sum_{j=0}^{k_p \ldots k_1-1} \nabla \| \nabla f_j(y_0) \|^2 + 
    \sum_{0\leq j \neq j' \leq k_p \ldots k_1-1} \nabla \langle \nabla f_j(y_0), \nabla f_{j'}(y_0) \rangle \right] = \\
    \nonumber = r k_n \ldots k_{p+1}\frac{1}{k_p^2 \ldots k_1^2} \left( k_p \ldots k_1 ANG(y_0) + k_p \ldots k_1 (k_p \ldots k_1-1) AIG(y_0) \right) = \\
    \nonumber = r k_n \ldots k_{p+1}\frac{1}{k_p \ldots k_1} \left( ANG(y_0) + (k_p \ldots k_1-1) AIG(y_0) \right) = \\
    = \frac{r^2 k_n^2 \ldots k_{p+1}^2}{m} \left( ANG(y_0) + (k_p \ldots k_1-1) AIG(y_0) \right) \label{generalization: 33}
\end{gather}

Now, combining expression (\ref{generalization: 33}) for $ \mathbb{E} \left[ \sum_{i=0}^{r k_n \ldots k_{p+1}-1} \nabla \| \nabla f_i^{(k_p \ldots k_1)} (y_0) \|^2 \right]$ with the coefficient before this term in (\ref{generalization: 32}), we obtain:

\begin{gather}
    \nonumber \frac{\gamma^2 \alpha_n \ldots \alpha_p (1 - \alpha_p) \alpha_{p-1}^2 \ldots \alpha_1^2 k_p^2 \ldots k_1^2}{4}
    \mathbb{E} \left[ \sum_{i=0}^{r k_n \ldots k_{p+1}-1} \nabla \| \nabla f_i^{(k_p \ldots k_1)} (y_0) \|^2 \right] = \\
    \nonumber = \frac{\gamma^2 \alpha_n \ldots \alpha_p (1 - \alpha_p) \alpha_{p-1}^2 \ldots \alpha_1^2 k_p^2 \ldots k_1^2}{4} \cdot
    \frac{r^2 k_n^2 \ldots k_{p+1}^2}{m} \left( ANG(y_0) + (k_p \ldots k_1-1) AIG(y_0) \right) = \\
    = \frac{\gamma^2 \beta m}{4} (1-\alpha_p) \alpha_{p-1} \ldots \alpha_1 \left( ANG(y_0) + (k_p \ldots k_1-1) AIG(y_0) \right) \label{generalization: 34}
\end{gather}

For the last term in (\ref{generalization: 32}), directly by the definition of $ANG(y)$ we get

\begin{gather}
    \frac{\gamma^2 \beta}{4} \mathbb{E} \left[ \sum_{i=0}^{m-1} \nabla \| \nabla f_i(y_0) \|^2 \right] =
    \frac{\gamma^2 \beta m}{4} ANG(y_0) \label{generalization: 35}
\end{gather}

Substituting (\ref{generalization: 34}) and (\ref{generalization: 35}) in (\ref{generalization: 32}), we have:

\begin{gather}
    \nonumber \mathbb{E}[y_r] = y_0
    - \gamma \beta m \nabla f(y_0)
    + \frac{\gamma^2 \beta^2 m^2}{4} \nabla \| \nabla f(y_0) \|^2 + \\
    \nonumber + \frac{\gamma^2 \beta m}{4} \sum_{p=1}^n (1-\alpha_p) \alpha_{p-1} \ldots \alpha_1 \left( ANG(y_0) + (k_p \ldots k_1-1) AIG(y_0) \right) - \\
    - \frac{\gamma^2 \beta m}{4} ANG(y_0)
    + \mathcal{O}(\gamma^3) \label{generalization: 36}
\end{gather}

Let us collect the coefficients before $ANG(y_0)$ and $AIG(y_0)$.

For $ANG(y_0)$: $- \frac{\gamma^2 \beta m}{4} + \frac{\gamma^2 \beta m}{4} \sum_{p=1}^n (1 - \alpha_p) \alpha_{p-1} \ldots \alpha_1 = - \frac{\gamma^2 \beta m}{4} + \frac{\gamma^2 \beta m}{4} \sum_{p=1}^n (\alpha_{p-1} \ldots \alpha_1 - \alpha_p \ldots \alpha_1)$ = $- \frac{\gamma^2 \beta m}{4} + \frac{\gamma^2 \beta m}{4} (1 - \alpha_n \ldots \alpha_1)$ = $- \frac{\gamma^2 \beta m}{4} + \frac{\gamma^2 \beta m}{4} (1 - \beta)$ = $- \frac{\gamma^2 \beta m}{4} + \frac{\gamma^2 \beta m}{4} - \frac{\gamma^2 \beta^2 m}{4}$ = $- \frac{\gamma^2 \beta^2 m}{4}$.

For $AIG(y_0)$: $\frac{\gamma^2 \beta m}{4} \sum_{p=1}^n (1-\alpha_p) \alpha_{p-1} \ldots \alpha_1 (k_p \ldots k_1-1)$.

Hence, (\ref{generalization: 36}) can be simplified in the following way:

\begin{gather}
    \nonumber \mathbb{E}[y_r] = y_0
    - \gamma \beta m \nabla f(y_0)
    + \frac{\gamma^2 \beta^2 m^2}{4} \nabla \| \nabla f(y_0) \|^2 - \\
    - \frac{\gamma^2 \beta^2 m}{4} ANG(y_0)
    + \frac{\gamma^2 \beta m}{4} \sum_{p=1}^n (1-\alpha_p) \alpha_{p-1} \ldots \alpha_1 (k_p \ldots k_1-1) AIG(y_0)
    + \mathcal{O}(\gamma^3) \label{generalization: 37}
\end{gather}

Now, substituting $\epsilon = \gamma \beta m$ in (\ref{generalization: 13}) (we instantly get $h(y) = - \nabla f(y)$ from equating the first order terms, so it was directly substituted in the formula below), we obtain that continuous solution of the ODE for the modified flow satisfies:

\begin{gather}
    y(\gamma \beta m) = y_0
    - \gamma \beta m \nabla f(y_0)
    + \gamma^2 \beta^2 m^2 \left( h_1(y_0) + \frac{1}{4} \nabla \| \nabla f(y_0) \|^2 \right) 
    + \mathcal{O}(\gamma^3)\label{generalization: 38}
\end{gather}

Equating the right-hand sides of (\ref{generalization: 37}) and (\ref{generalization: 38}), we can find $h_1(y)$:

\begin{gather}
    h_1(y) = -\frac{1}{4m} ANG(y) + \frac{1}{4\beta m} \sum_{p=1}^n (1-\alpha_p) \alpha_{p-1} \ldots \alpha_1 (k_p \ldots k_1-1) AIG(y)
\end{gather}

Thus, we recover the formula of the modified flow for $n$-layers Lookahead:

\begin{gather}
    \nonumber \Tilde{h}_{LA-n}(y) = -\nabla f(y) + \gamma \beta m h_1(y) + \mathcal{O}(\gamma^2) = \\
    = -\nabla f(y)
    -\frac{\gamma \beta}{4} ANG(y)
    + \frac{\gamma}{4} \sum_{p=1}^n (1-\alpha_p) \alpha_{p-1} \ldots \alpha_1 (k_p \ldots k_1-1) AIG(y)
    + \mathcal{O}(\gamma^2)
\end{gather}

Finally, from $\Tilde{h}_{LA-n}(y) = -\nabla \Tilde{f}_{LA-n}(y)$, we get:

\begin{gather}
    \Tilde{f}_{LA-n}(y) = f(y)
    + \frac{\gamma \beta}{4} AN(y)
    - \frac{\gamma}{4} \sum_{p=1}^n (1-\alpha_p) \alpha_{p-1} \ldots \alpha_1 (k_p \ldots k_1-1) AI(y)
    + \mathcal{O}(\gamma^2)
\end{gather}

and unrolling $\beta = \alpha_n \ldots \alpha_1$, we get the desired result.

\end{proof}

\subsection{Multilayer Lookahead Preserves Linear Convergence Rate} \label{MLA linear rate}

In this section, we show that if the inner optimizer has a linear convergence rate (either in terms of distance to the optimal point or in terms of the function value), then Lookahead (with any number of layers) preserves linear convergence rate.

Recall the update rule of the Lookahead:

\begin{gather}
    \nonumber x_{t,i+1} = x_{t,i} - \gamma g_{t,i}, \: i=0, ..., k-1 \\
    y_{t+1} = (1 - \alpha) y_t + \alpha x_{t,k} \label{59}
\end{gather}

In the following, we denote by $f$ the function to be minimized, $x^*$ its global minimum, and $f^*$ its minimum value.

\begin{claim}
Suppose that the inner optimizer of Lookahead has a linear convergence rate in terms of the distance to the optimal point; that is, it gives a sequence of iterations $(x_t, t\geq0)$, such that $\lVert x_{t+1} - x^* \rVert \leq c \lVert x_t - x^*\rVert$, $c < 1$. Then Lookahead also has a linear convergence rate.
\end{claim}

\begin{proof}

From the linear convergence speed of the inner optimizer, we  have:

\begin{gather*}
    \lVert x_{t,k} - x^* \rVert \leq c^k \lVert x_{t,0} - x^* \rVert = c^k \lVert y_t - x^* \rVert
\end{gather*}

Hence,

\begin{gather*}
    \lVert y_{t+1} - x^* \rVert \leq \lVert (1-\alpha) y_t + \alpha x_{t,k} - x^* \rVert \leq (1-\alpha) \lVert y_t - x^* \rVert + \alpha \lVert x_{t,k} - x^*\rVert \leq \\
    \leq (1-\alpha) \lVert y_t - x^* \rVert + \alpha c^k \lVert y_t - x^* \rVert = (1-\alpha + \alpha c^k) \lVert y_t - x^* \rVert
\end{gather*}

Thus, we get:

\begin{gather*}
    \lVert y_{t+1} - x^* \rVert \leq (1-\alpha(1 - c^k)) \lVert y_t - x^* \rVert
\end{gather*}

and $(1-\alpha(1 - c^k)) < 1$ since $c^k < 1$.

\end{proof}

\begin{claim} \label{lin_conv:lemma1}
Let $f$ be a convex function and assume that the inner optimizer of Lookahead has a linear rate of convergence in terms of function value; that is, it gives a sequence of iterations $(x_t, t\geq0)$, such that $f(x_{t+1}) - f^* \leq c (f(x_t) - f(x^*))$, $c < 1$. Then Lookahead also has a linear convergence rate.
\end{claim}

\begin{proof}
From the linear convergence speed of the inner optimizer, we  have:

\begin{gather*}
    f(x_{t,k}) - f^* \leq c^k (f(x_{t,0}) - f^*) = c^k (f(y_t) - f^*)
\end{gather*}

Hence,

\begin{gather*}
    f(y_{t+1}) - f^* = f((1-\alpha)y_t + \alpha x_{t,k}) - f^* \leq |\text{using convexity}| \leq (1-\alpha)f(y_t) + \alpha f(x_{t,k}) - f^* = \\
    = (1-\alpha)(f(y_t) - f^*) + \alpha (f(x_{t,k}) - f^*) \leq (1-\alpha + \alpha c^k)(f(y_t) - f^*)
\end{gather*}

Thus, we get:

\begin{gather*}
    f(y_{t+1}) - f^* \leq (1-\alpha (1 - c^k))(f(y_t) - f^*)
\end{gather*}

and $(1-\alpha (1 - c^k)) < 1$ since $c^k < 1$.

\end{proof}

\begin{remark}
Both claims remain valid when the inner optimizer is stochastic and the linear convergence holds in expectation. The proof for this case is identical, thus omitted.
\end{remark}

\begin{remark}
Since we can consider Multilayer Lookahead with $n$ layers as a 1-layer Lookahead with $n-1$-layers Lookahead as the inner optimizer, we can extend the result of both claims to the Multilayer Lookahead by induction over the number of layers. However, with each new layer, the constant of convergence degrades.
\end{remark}

\begin{remark}
As an application of Claim \ref{lin_conv:lemma1}, we get a linear convergence of Lookahead with gradient descent as the inner optimizer for smooth and strongly convex functions.
\end{remark}

Both lemmas prove linear convergence of Lookahead with the same constant $1-\alpha(1 - c^k)$ (but in different senses). Since $c^k < 1$, the constant of convergence decreases as alpha increases, so it achieves the minimum value for $\alpha = 1$, that is when Lookahead degenerates to its base optimizer. We should have expected this because we made only minimal assumptions on the loss function and the inner optimizer.

\section{ADDITIONAL EXPERIMENTS}

\subsection{Additional details on classification on CIFAR-10 and CIFAR-100} \label{cifar experiments details}

Here, we present additional details for the Section \ref{sec: experiments}.

For both CIFAR-10 and CIFAR-100, for Multilayer Lookahead we fixed weight decay to be optimal for SGD. The other parameters (for both CIFAR-10 and CIFAR-100) are:

LA-1 (Lookahead): 
$\alpha \in \{0.3, 0.5, 0.7\}$, 
$k \in \{5, 10\}$.

LA-2: 
fixed $k = (5,5)$, 
$\alpha = (\alpha_1, \alpha_2)$ for
$\alpha_1 \in \{0.5, 0.6, 0.7, 0.8\}$,
$\alpha_2 \in \{0.5, 0.6, 0.7, 0.8\}$
if $\alpha_1 \leq \alpha_2$.

LA-3: 
fixed $k = (5,5,5)$, 
$\alpha = (\alpha_1, \alpha_2, \alpha_3)$ for
$\alpha_1 \in \{0.6\}$,
$\alpha_2 \in \{0.7, 0.75, 0.8, 0.85\}$,
$\alpha_3 \in \{0.75, 0.8, 0.85\}$
if $\alpha_1 \leq \alpha_2 \leq \alpha_3$.

LA-4:
fixed $k = (5,5,5,5)$, 
$\alpha = (\alpha_1, \alpha_2, \alpha_3, \alpha_4)$ for
$\alpha_1 \in \{0.6\}$,
$\alpha_2 \in \{0.7, 0.75, 0.8, 0.85\}$,
$\alpha_3 \in \{0.75, 0.8, 0.85\}$,
$\alpha_4 \in \{0.75, 0.8, 0.85\}$
if $\alpha_1 \leq \alpha_2 \leq \alpha_3 \leq \alpha_4$.

and $lr \in \{0.1, 0.2, 0.3\}$ for CIFAR-10, 
$lr \in \{0.01, 0.03, 0.05\}$ for CIFAR-100.

For each number of layers, we compared all configurations using random seed = 42 and selected 2-3 the best candidates based on this seed. Then, we compared the selected candidates by the average performance on seed $\in \{1,2,3\}$ to select the best configuration. Finally, we compared the best configurations for each optimizer based on the average performance on seed $\in \{4,5,6\}$. See additional plots for comparing best configurations in Figure \ref{fig:CIFAR-10 additional plots} (CIFAR-10) and Figure \ref{fig:CIFAR-100 additional plots} (CIFAR-100).

We emphasize that the performance of Multilayer Lookahead for both CIFAR-10 and CIFAR-100 was robust towards the choice of parameters ($\alpha$, $k$) in the considered range.

\begin{figure}[h]
     \centering
     \hfill
     \begin{subfigure}[b]{0.45\textwidth}
         \centering
         \includegraphics[width=\textwidth]{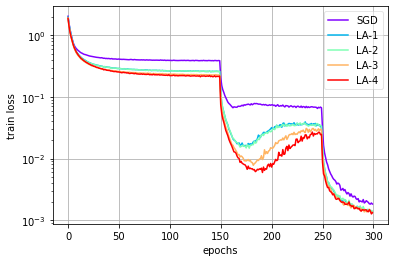}
         \caption{\small Train Loss}
     \end{subfigure}
     \hfill
     \begin{subfigure}[b]{0.45\textwidth}
         \centering
         \includegraphics[width=\textwidth]{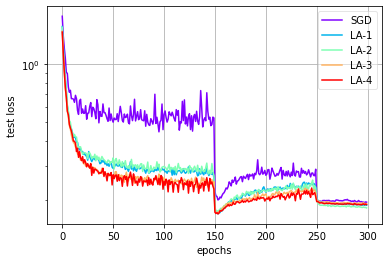}
         \caption{\small Test Loss}
     \end{subfigure}
     \hfill
     \begin{subfigure}[b]{0.45\textwidth}
         \centering
         \includegraphics[width=\textwidth]{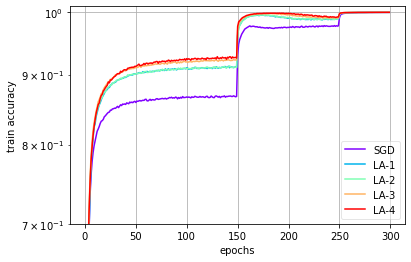}
         \caption{\small Train Accuracy}
     \end{subfigure}
    \caption{\small Additional Training Plots for CIFAR-10 Classification}
    \label{fig:CIFAR-10 additional plots}
\end{figure}

\begin{figure}[h]
     \centering
     \hfill
     \begin{subfigure}[b]{0.45\textwidth}
         \centering
         \includegraphics[width=\textwidth]{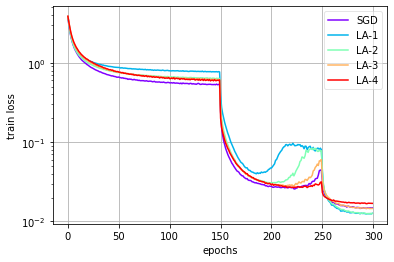}
         \caption{\small Train Loss}
     \end{subfigure}
     \hfill
     \begin{subfigure}[b]{0.45\textwidth}
         \centering
         \includegraphics[width=\textwidth]{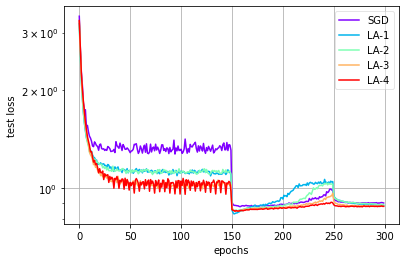}
         \caption{\small Test Loss}
     \end{subfigure}
     \hfill
     \begin{subfigure}[b]{0.45\textwidth}
         \centering
         \includegraphics[width=\textwidth]{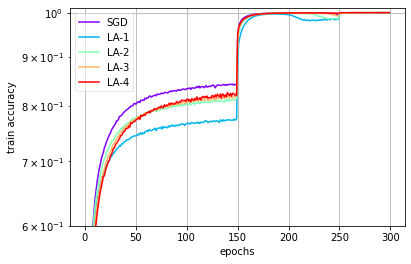}
         \caption{\small Train Accuracy}
     \end{subfigure}
    \caption{\small Additional Training Plots for CIFAR-100 Classification}
    \label{fig:CIFAR-100 additional plots}
\end{figure}

\subsection{Training GANs on MNIST} \label{GANs on MNIST}

In this section, we present results of training GANs for digits generation on MNIST dataset. We chose Deep Convolutional GANs (DCGANs) model \citep{radford2015unsupervised}), where both Discriminator and Generator entirely consist of Convolutional (Transposed Convolutional) layers. You can find the architectures for Generator and Discriminator in Figure  \ref{fig:GANs models}. The Generator takes a noise vector of $128\times1\times1$ size from the normal distribution as an input. We used the non-saturating loss function for GANs \citep{goodfellow2014generative} and Adam optimizer \citep{kingma2014adam}, which also served as the inner optimizer for Multilayer Lookahead. We fixed Adam parameters to be $\beta_1 = 0.5$, $\beta_2 = 0.999$, and tried learning rate in $\{0.0002, 0.001\}$. Such choice of parameters (with learning rate 0.0002) was suggested by \citet{radford2015unsupervised} as the default choice for DCGANs. However, we expected the optimal learning rate for Multilayer Lookahead to be larger. We used the same number of steps $k_i=5$ and synchronization parameter $\alpha_i=0.5$ for each layer of the Multilayer Lookahead, and tried number of layers from 1 to 6.

\begin{figure}
     \centering
     \begin{subfigure}[b]{0.45\textwidth}
         \centering
         \includegraphics[width=\textwidth]{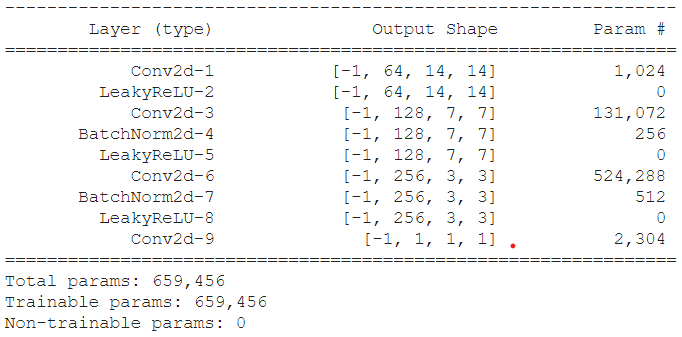}
         \caption{\small Discriminator Architecture}
     \end{subfigure}
     \hfill
     \begin{subfigure}[b]{0.45\textwidth}
         \centering
         \includegraphics[width=\textwidth]{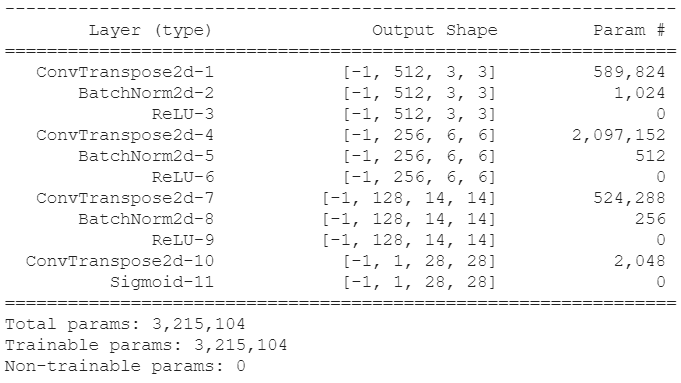}
         \caption{\small Generator Architecture}
     \end{subfigure}
    \caption{\small Models Used for GANs}
    \label{fig:GANs models}
\end{figure}

For each optimizer, we trained our models on 100 epochs using batch size = 50. We found the learning rate = 0.0002 to be better for Adam and Multilayer Lookahead with $\leq 3$ layer, and the learning rate = 0.001 to be better for the number of layers from 4 to 6. For comparing the models, we used the Inception Score metric \citep{salimans2016improved}. In the report, we do not use another popular metric - Frechet Inception Distance, because we found this metric to be unstable for MNIST. Finally, we compared the optimizers with tuned learning rate in Figure \ref{fig: MNIST_IS}. Our preliminary analysis shows that Multilayer Lookahead can significantly improve both Adam and Lookahead.

\begin{figure}[h]
\centering
\includegraphics[width=0.45\textwidth]{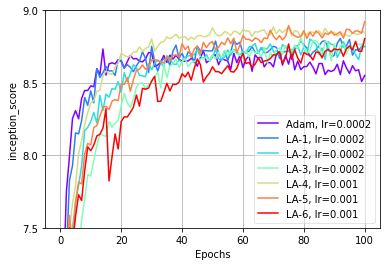}
\caption{\small Comparing Adam and Multilayer Lookahead with Adam as base optimizer on training GANs on MNIST. For each optimizer, the best learning rate from \{0.0002, 0.001\} is chosen.}
\label{fig: MNIST_IS}
\end{figure}

\end{document}